\documentclass{article}

\usepackage[final]{neurips_2024}
\usepackage{amsmath}
\usepackage{amssymb}
\usepackage{mathtools}
\usepackage{amsthm}
\usepackage{nicefrac}
\usepackage{multirow}
\usepackage{tikz-cd} 
\usepackage{subcaption}
\usepackage{multicol,caption}

\usepackage{wrapfig}
\usepackage{algorithm}
\usepackage[noend]{algorithmic}

\usepackage[flushleft]{threeparttable} 
\usepackage{colortbl}
\definecolor{bgcolor}{rgb}{0.8,1,1}
\definecolor{bgcolor2}{rgb}{0.8,1,0.8}
\definecolor{niceblue}{rgb}{0.0,0.19,0.56}

\usepackage{hyperref}
\hypersetup{colorlinks,linkcolor={blue},citecolor={blue},urlcolor={blue}}

\usepackage{pifont}
\definecolor{PineGreen}{RGB}{0,110,51}
\definecolor{BrickRed}{RGB}{143,20,2}
\newcommand{\cmark}{{\color{PineGreen}\ding{51}}}%
\newcommand{\xmark}{{\color{BrickRed}\ding{55}}}%

\usepackage{thmtools}

\newcommand{\cD}{{\cal D}}
\newcommand{\Exp}[1]{\mathbb{E} \left[ #1\right]}
\newcommand{\Expt}[1]{\mathbb{E}_{t} \left[ #1\right]}

\providecommand{\mycomment}[3]{\todo[inline,caption={},color=#3!20]{\textbf{#1: }#2}}%
\providecommand{\inlinecomment}[3]{{\color{#1}#2: #3}}
\newcommand\commenter[2]%
{%
  \expandafter\newcommand\csname i#1\endcsname[1]{\inlinecomment{#2}{#1}{##1}}
  \expandafter\newcommand\csname #1\endcsname[1]{\mycomment{#1}{##1}{#2}}
}

\usepackage[colorinlistoftodos,bordercolor=orange,backgroundcolor=orange!20,linecolor=orange,textsize=scriptsize]{todonotes}

\commenter{Nicolas}{yellow}
\commenter{Eduard}{green}
\commenter{Sayantan}{white}

\definecolor{shadecolor}{gray}{0.9}
\declaretheoremstyle[
headfont=\normalfont\bfseries,
notefont=\mdseries, notebraces={(}{)},
bodyfont=\normalfont,
postheadspace=0.5em,
spaceabove=1pt,
mdframed={
  skipabove=8pt,
  skipbelow=4pt,
  hidealllines=true,
  backgroundcolor={shadecolor},
  innerleftmargin=4pt,
  innerrightmargin=4pt,
  innertopmargin=8pt}
]{shaded}

\declaretheorem[style=shaded,within=section]{definition}
\declaretheorem[style=shaded,sibling=definition]{theorem}
\declaretheorem[style=shaded,sibling=definition]{proposition}
\declaretheorem[style=shaded,sibling=definition]{assumption}

\declaretheorem[style=shaded,sibling=definition]{lemma}

\usepackage[capitalize,noabbrev]{cleveref}

\usepackage{xspace}

\newcommand{\algname}[1]{{\color{PineGreen}\sf  #1}\xspace}

\newcommand{\R}{\mathbb{R}}

\newcommand{\E}{\mathbb{E}}
\newcommand{\cO}{{\cal O}}

\newcommand{\hw}{\hat{w}}

\newcommand{\la}{\left\langle}

\newcommand{\ra}{\right\rangle}

\newcommand{\eqdef}{:=}

\usepackage{natbib}
\DeclareMathOperator{\diag}{diag}

\usepackage[utf8]{inputenc} 
\usepackage[T1]{fontenc}    
\usepackage{hyperref}       
\usepackage{url}            
\usepackage{booktabs}       
\usepackage{amsfonts}       
\usepackage{nicefrac}       
\usepackage{microtype}      
\usepackage{xcolor}         

\title{Remove that Square Root: A New Efficient\\ Scale-Invariant Version of \algname{AdaGrad}}

\author{%
  Sayantan Choudhury\thanks{\text{Part of this work was done when S. Choudhury was an intern at MBZUAI, UAE.}} \\
  MBZUAI \& Johns Hopkins University \\
  \And
  Nazarii Tupitsa \\
  MBZUAI \&  Innopolis University \\
  \And
  Nicolas Loizou \\
  Johns Hopkins University \\
  \And
  Samuel Horv\'ath\\
  MBZUAI\\
  \And
  Martin Tak\'a\v{c}\\
  MBZUAI\\
  \And
  Eduard Gorbunov\\
  MBZUAI
}

\begin{document}

\maketitle

\begin{abstract}
  Adaptive methods are extremely popular in machine learning as they make learning rate tuning less expensive. This paper introduces a novel optimization algorithm named \algname{KATE}, which presents a scale-invariant adaptation of the well-known \algname{AdaGrad} algorithm. We prove the scale-invariance of \algname{KATE} for the case of Generalized Linear Models. Moreover, for general smooth non-convex problems, we establish a convergence rate of  $\cO(\nicefrac{\log T}{\sqrt{T}})$ for \algname{KATE}, matching the best-known ones for \algname{AdaGrad} and \algname{Adam}. We also compare \algname{KATE} to other state-of-the-art adaptive algorithms \algname{Adam} and \algname{AdaGrad} in numerical experiments with different problems, including complex machine learning tasks like image classification and text classification on real data. The results indicate that \algname{KATE} consistently outperforms \algname{AdaGrad} and matches/surpasses the performance of \algname{Adam} in all considered scenarios.
\end{abstract}

\section{Introduction}
In this work, we consider the following unconstrained optimization problem:
\begin{equation}\label{eq:problem}
{\min}_{w \in \R^d} f(w), 
\end{equation}
where $f:\R^d \to \R$ is a $L$-smooth and generally non-convex function. In particular, we are interested in the situations when the objective has either expectation $f(w) = \mathbb{E}_{\xi\sim \cD}[f_{\xi}(w)]$ or finite-sum $f(w) = \frac{1}{n}\sum_{i=1}^n f_i(w)$ form. Such minimization problems are crucial in machine learning, where $w$ corresponds to the model parameters. Solving these problems with stochastic gradient-based optimizers has gained much interest owing to their wider applicability and low computational cost. Stochastic Gradient Descent (\algname{\textcolor{PineGreen}{SGD}})~\citep{robbins1951stochastic} and similar algorithms require the knowledge of parameters like $L$ for convergence and are very sensitive to the choice of the stepsize in general. Therefore, \algname{\textcolor{PineGreen}{SGD}} requires hyperparameter tuning, which can be computationally expensive. To address these issues, it is common practice to use adaptive variants of stochastic gradient-based methods that can converge without knowing the function's structure. 

There exist many adaptive algorithms such as \algname{\textcolor{PineGreen}{AdaGrad}}~\citep{duchi2011adaptive}, \algname{\textcolor{PineGreen}{Adam}}~\citep{kingma2014adam}, 
\algname{\textcolor{PineGreen}{AMSGrad}}~\citep{reddi2019convergence}, \algname{D-Adaptation} \citep{defazio2023learning}, \algname{Prodigy} \citep{mishchenko2023prodigy},
\algname{AI-SARAH} \citep{shi2023aisarah}
and their variants. These adaptive techniques are capable of updating their step sizes on the fly. For instance, the \algname{\textcolor{PineGreen}{AdaGrad}} method determines its step sizes using a cumulative sum of the coordinate-wise squared (stochastic) gradient of all the previous iterates: 
\begin{eqnarray}\label{eq:AdaGrad}
    \text{\algname{\textcolor{PineGreen}{AdaGrad}}:} \;  w_{t+1} = w_t - \frac{\beta g_t}{\sqrt{\text{ diag} \left(\Delta I + \sum\limits_{\tau = 1}^t g_{\tau} g_{\tau}^{\top} \right)}},
\end{eqnarray}
where $g_t$ represents an unbiased estimator of $\nabla f(w_t)$, i.e., $\Exp{g_t\mid w_t} = \nabla f(w_t)$, $\text{diag}(M) \in \R^d$ is a vector of diagonal elements of matrix $M \in \R^{d\times d}$, $\Delta > 0$, and the division by vector is done component-wise. \citet{ward2020adagrad} has shown that this method achieves a convergence rate of $\mathcal{O}\left(\nicefrac{\log T}{\sqrt{T}}\right)$ for smooth functions, similar to \algname{\textcolor{PineGreen}{SGD}}, without prior knowledge of the functions' parameters. However, the performance of \algname{\textcolor{PineGreen}{AdaGrad}} deteriorates when applied to data that may exhibit poor scaling or ill-conditioning. In this work, we propose a novel algorithm, \algname{\textcolor{PineGreen}{KATE}}, to address the issues of poor data scaling. \algname{\textcolor{PineGreen}{KATE}} is also a stochastic adaptive algorithm that can achieve a convergence rate of $\mathcal{O}\left( \nicefrac{\log T}{\sqrt{T}}\right)$ for smooth non-convex functions in terms of $\min_{t \in [T]} \Exp{\left\| \nabla f(w_t)\right\|}^2$.

\subsection{Related Work}

A significant amount of research has been done on adaptive methods over the years, including \algname{\textcolor{PineGreen}{AdaGrad}}~\citep{duchi2011adaptive, mcmahan2010adaptive}, \algname{\textcolor{PineGreen}{AMSGrad}}~\citep{reddi2019convergence}, \algname{\textcolor{PineGreen}{RMSProp}}~\citep{tieleman2012rmsprop}, 
\algname{\textcolor{PineGreen}{AI-SARAH}}~\citep{shi2023aisarah},
and \algname{\textcolor{PineGreen}{Adam}}~\citep{kingma2014adam}. However, all these works assume that the optimization problem is contained in a bounded set. To address this issue, \citet{li2019convergence} proposes a variant of the \algname{\textcolor{PineGreen}{AdaGrad}} algorithm, which does not use the gradient of the last iterate (this makes the step sizes of $t$-th iteration conditionally independent of $g_t$) for computing the step sizes and proves convergence for the unbounded domain. 

Each of these works considers a vector of step sizes for each coefficient.
\citet{duchi2011adaptive} and \citet{mcmahan2010adaptive} simultaneously proposed the original \algname{\textcolor{PineGreen}{AdaGrad}} algorithm. However, \citet{mcmahan2010adaptive} was the first to consider the vanilla scalar form of \algname{\textcolor{PineGreen}{AdaGrad}}, known as
\begin{eqnarray}
    \text{\algname{\textcolor{PineGreen}{AdaGradNorm}}: }  w_{t+1} = w_t - \frac{\beta g_t}{\sqrt{\Delta+\sum_{\tau = 0}^t \left\| g_{\tau}\right\|^2}}.
\end{eqnarray}
Later, \citet{ward2020adagrad} analyzed \algname{\textcolor{PineGreen}{AdaGradNorm}} for minimizing smooth non-convex functions. In a follow-up study, \citet{xie2020linear} proves a linear convergence of \algname{\textcolor{PineGreen}{AdaGradNorm}} for strongly convex functions. Recently, \citet{liu2022convergence} analyzed \algname{\textcolor{PineGreen}{AdaGradNorm}} for solving smooth convex functions without the bounded domain assumption. Moreover, \citet{liu2022convergence} extends the convergence guarantees of \algname{\textcolor{PineGreen}{AdaGradNorm}} to quasar-convex functions~\footnote{$f$ satisfy $f^* \geq f(w) + \frac{1}{\zeta} \la f(w), w^* - w\ra$ for some $\zeta \in (0, 1]$ where $w^* \in \text{argmin}_w f(w)$.} using the function value gap. \citet{orabona2015generalized} introduce the notion of scale-invariance, which is a special case of affine invariance \citep{nesterov1994interior, nesterov2018lectures, d2018optimal}, propose a scale-invariant version of \algname{\textcolor{PineGreen}{AdaGrad}} for online convex optimization for generalized linear models, and prove $\cO(\sqrt{T})$ regret bounds in this setup.

Recently, \citet{defazio2023learning} introduced the \algname{D-Adaptation} method, which has gathered considerable attention due to its promising empirical performances. In order to choose the adaptive step size optimally, one requires knowledge of the initial distance from the solution, i.e., $D \coloneqq \left\| w_0 -w_*\right\|$ where $w_* \in \text{argmin}_{w \in \R^d} f(w)$. The \algname{D-Adaptation} method works by maintaining an estimate of $D$ and the stepsize choice in this case is $\nicefrac{d_t}{\sqrt{\sum_{\tau = 0}^t \|g_{\tau}\|^2}}$ for the $t$-th iteration (here $d_t$ is an estimate of $D$). \citet{mishchenko2023prodigy} further modifies the algorithm in a follow-up work and introduces \algname{Prodigy} (with stepsize choice $\nicefrac{d^2_t}{\sqrt{\sum_{\tau = 0}^t d_{\tau}^2\|g_{\tau}\|^2}}$) to improve the convergence speed.

\begin{table}[tbp]\label{table:VR}
    \centering
    \caption{
    Summary of convergence guarantees for closely-related adaptive algorithms to solve \textit{smooth non-convex stochastic} optimization problems. Convergence rates are given in terms of $\min_{t \in [T]} \Exp{\left\| \nabla f(w_t)\right\|}^2$. We highlight \algname{KATE}'s \textit{scale-invariance} property for problems of type \eqref{eq:GLM_opt_unscaled}. 
    }
    \label{table:compare_ada}
    \renewcommand{\arraystretch}{2}
        \begin{tabular}{c   c   c }
            \toprule
            \multicolumn{1}{c }{Algorithm}
            & 
            \multicolumn{1}{c }{
                \shortstack{Convergence rate}
            }
            & 
            \multicolumn{1}{c}{
                \shortstack{Scale invariance}
            }
            \\
           \midrule
            \multicolumn{1}{c }{\algname{AdaGradNorm} \citep{ward2020adagrad}
            }
            & $\mathcal{O}\left( \nicefrac{\log T}{\sqrt{T}}\right)$ & \xmark
            \\
            \hline
            \algname{AdaGrad}\citep{defossez2020simple} & $\mathcal{O}\left( \nicefrac{\log T}{\sqrt{T}}\right)$ & \xmark
            \\
            \hline 
            \begin{tabular}{c}
                \algname{Adam} \citep{defossez2020simple}
            \end{tabular} & $\mathcal{O} \left( \nicefrac{\log T}{\sqrt{T}} \right)$ & \xmark
            \\
            \hline
            \cellcolor{green!20}
            \begin{tabular}{c}
                \algname{\textcolor{PineGreen}{KATE}} (this work)
            \end{tabular} & \cellcolor{green!20}
            $\mathcal{O} \left(\nicefrac{\log T}{\sqrt{T}} \right)$ & \cellcolor{green!20}
            \cmark 
            \\
            \bottomrule
       \end{tabular}
     
\end{table}
Another exciting line of work on adaptive methods is Polyak stepsizes. \citet{polyak1969minimization} first proposed Polyak stepsizes for subgradient methods, and recently, the stochastic version (also known as \algname{SPS}) was introduced by \citet{oberman2019stochastic,loizou2021stochastic,abdukhakimov2023stochastic,abdukhakimov2023sania,li2022sp2}
and \citet{gower2021stochastic}. For a finite sum problem $\min_{w \in \R^d} f(w) \eqdef \frac{1}{n} \sum_{i = 1}^n f_i(w)$, \citet{loizou2021stochastic} uses $\frac{f_i(w_t) - f_i^*}{c \|\nabla f_i (w_t)\|^2}$ as their stepsize choices (here $f_i^* \eqdef \min_{w \in \R^d} f_i(w)$), while \citet{oberman2019stochastic} uses $\frac{2(f(w_t) - f^*)}{\Exp{\|\nabla f_i(w_t)\|^2}}$ for $k$-th iteration. However, these methods are impractical when $f^*$ or $f_i^*$ is unknown. Following its introduction, several variants of the \algname{SPS} algorithm emerged \citep{li2022sp2, d2021stochastic}. Lately, \citet{orvieto2022dynamics} tackled the issues with unknown $f_i^*$ and developed a truly adaptive variant. In practice, the \algname{SPS} method shows excellent empirical performance on overparameterized deep learning models (which satisfy the interpolation condition i.e. $f_i^* = 0, \  \forall i \in [n]$) \citep{loizou2021stochastic}. 
\subsection{Main Contribution}
Our main contributions are summarized below.
\paragraph{$\bullet$ \algname{KATE}: new scale-invariant version of \algname{AdaGrad}.} We propose a new method called \algname{KATE} that can be seen as a version of \algname{AdaGrad}, which does not use a square root in the denominator of the stepsize. To compensate for this change, we introduce a new sequence defining the numerator of the stepsize. We prove that \algname{KATE} is scale-invariant for generalized linear models: if the starting point is zero, then the loss values (and training and test accuracies in the case of classification) at points generated by \algname{KATE} are independent of the data scaling (Proposition~\ref{prop:invariance}), meaning that the speed of convergence of \algname{KATE} is the same as for the best scaling of the data.
\paragraph{$\bullet$ Convergence for smooth non-convex problems.} We prove that for smooth non-convex problems with noise having bounded variance \algname{KATE} has $\cO(\nicefrac{\log(T)}{\sqrt{T}})$ convergence rate (Theorem \ref{theorem:stoch_nonconvex}), matching the best-known rates for \algname{AdaGrad} and \algname{Adam} \citep{defossez2020simple}.
\paragraph{$\bullet$ Numerical experiments.} We empirically illustrate the scale-invariance of \algname{KATE} on the logistic regression task and test its performance on logistic regression (see Section \ref{section:logistic_reg}), image classification, and text classification problems (see Section \ref{section:neuralnet}). In all the considered scenarios, \algname{KATE} outperforms \algname{AdaGrad} and works either better or comparable to \algname{Adam}.
\subsection{Notation}
We denote the set $\{1, 2, \cdots, n\}$ as $[n]$. For a vector $a \in \R^d$, $a[k]$ is the $k$-th coordinate of $a$ and $a^2$ represents the element-wise suqare of $a$, i.e., $a^2[k] = (a[k])^2$. For two vectors $a$ and $b$, $\frac{a}{b}$ stands for element-wise division of $a$ and $b$, i.e., $k$-th coordinate of $\frac{a}{b}$ is $\frac{a[k]}{b[k]}$. Given a function $h: \R^d \to \R$, we use $\nabla h \in \R^d$ to denote its gradient and $\nabla_k h$ to indicate the $k$-th component of $\nabla h$. Throughout the paper $\| \cdot \|$ represents the $\ell_2$-norm and $f_* = \inf_{w \in \R^d} f(w)$. Moreover, we use $\left\| w \right\|_A$ for a positive-definite matrix $A$ to define $\left\| w \right\|_A \coloneqq \sqrt{w^{\top}Aw}$. Furthermore, $\Exp{\cdot}$ denotes the total expectation while $\Expt{\cdot}$ denotes the conditional expectation conditioned on all iterates up to step $t$ i.e. $w_0, w_1,\ldots,w_t$.

\section{Motivation and Algorithm Design}\label{sec:motivation_alg}
We focus on solving the minimization problem \eqref{eq:problem} using a variant of \algname{\textcolor{PineGreen}{AdaGrad}}. We aim to design an algorithm that performs well, irrespective of how poorly the data is scaled. 
\paragraph{Generalized linear models.} Here, we consider the parameter estimation problem in generalized linear models (GLMs) \citep{nelder1972generalized, agresti2015foundations} using maximum likelihood estimation. GLMs are an extension of linear models and encompass several other valuable models, such as logistic \citep{hosmer2013applied} and Poisson regression \citep{frome1983analysis}, as special cases. The parameter estimation to fit GLM on dataset $\{x_i, y_i \}_{i=1}^n$ (where $x_i \in \R^d$ are feature vectors and $y_i$ are response variables) can be reformulated as 
\begin{eqnarray}\label{eq:GLM_opt_unscaled}
    \min_{w \in \R^d} f(w) \coloneqq \frac{1}{n} \sum_{i = 1}^n \varphi_i \left(x_i^{\top}w \right)
\end{eqnarray}
for differentiable functions $\varphi_i : \R \to \R$ \citep{shalev2014understanding,
nguyen2017stochastic,
takavc2013mini,he2018dual,
chezhegov2024local}.
For example, the linear regression on data $\{x_i, y_i \}_{i=1}^n$ is equivalent to solving \eqref{eq:GLM_opt_unscaled} with $\varphi_i(z) = (z - y_i)^2$. Next, the choice of $\varphi_i$ for logistic regression is $\varphi_i(z) = \log \left(1 + \exp{\left( -y_i z\right)} \right)$. 
\paragraph{Scale-invariance.} Now consider the instances of fitting GLMs on two datasets $\{x_i, y_i \}_{i=1}^n$ and $\{Vx_i, y_i \}_{i=1}^n$, where $V \in \R^{d \times d}$ is a diagonal matrix with positive entries. Note that the second dataset is a scaled version of the first one where the $k$-th component of feature vectors $x_i$ are multiplied by a scalar $V_{kk}$. Then, the minimization problems corresponding to datasets $\{x_i, y_i \}_{i=1}^n$ and $\{Vx_i, y_i \}_{i=1}^n$ are \eqref{eq:GLM_opt_unscaled} and 
\begin{eqnarray}
     \min_{w \in \R^d} f^V(w) & \coloneqq & \frac{1}{n} 
     {\sum}_{i = 1}^n \varphi_i \left(x_i^{\top}Vw \right) \label{eq:GLM_opt_scaled},
\end{eqnarray}
respectively, for functions $\varphi_i$. In this work, \textit{we want to design an algorithm with equivalent performance for the problems \eqref{eq:GLM_opt_unscaled} and \eqref{eq:GLM_opt_scaled}}. If we can do that, the new algorithm's performance will not deteriorate for poorly scaled data, i.e., the method will be scale-invariant \citep{orabona2015generalized}, which is a special case of affine-invariance, see \citep{nesterov1994interior, nesterov2018lectures, d2018optimal}.
To develop such an algorithm, we replace the denominator of \algname{\textcolor{PineGreen}{AdaGrad}} step size with its square (remove the square root from the denominator), i.e., $\forall k \in [d]$
\begin{eqnarray}\label{eq:invariant_adagrad}
    w_{t+1}[k] & = & w_t[k] - \frac{\beta m_t[k]}{\sum_{\tau = 0}^t g_{\tau}^2[k]}g_t[k]
\end{eqnarray}
for some $m_t \in \R^d$.\footnote{Sequence $\{m_t\}_{t\geq 0}$ can depend on the problem but is assumed to be scale-invariant.}
The following proposition shows that this method \eqref{eq:invariant_adagrad} satisfies a scale-invariance property with respect to functional value. 

\begin{proposition}[Scale invariance]\label{prop:invariance}
    Suppose we solve problems \eqref{eq:GLM_opt_unscaled} and \eqref{eq:GLM_opt_scaled} using algorithm \eqref{eq:invariant_adagrad}. Then, the iterates $\hw_t$ and $\hw^V_t$ corresponding to \eqref{eq:GLM_opt_unscaled} and \eqref{eq:GLM_opt_scaled} follow: $\forall k \in [d]$
    \begin{eqnarray}
        \hw_{t+1}[k] &=& \hw_t[k] - \frac{\beta m_t[k]}{\textstyle{\sum}_{\tau = 0}^t g_{\tau}^2[k]} g_t[k], \label{eq:prop_invariance_update1}\\
        \quad \hw^V_{t+1}[k] &= & \hw^V_t[k] - \frac{\beta m_t[k]}{\textstyle{\sum}_{\tau = 0}^t \left(g_{\tau}^V[k] \right)^2} g^V_t[k] \label{eq:prop_invariance_update2}
    \end{eqnarray}
    with $g_{\tau} = \varphi'_{i_{\tau}}(x_{i_{\tau}}^{\top}\hw_{\tau}) x_{i_{\tau}}$ and $g^V_{\tau} = \varphi'_{i_{\tau}}(x_{i_{\tau}}^{\top}V \hw_{\tau}) Vx_{i_{\tau}}$ for $i_{\tau}$ chosen uniformly from $[n]$, $\tau = 0,1,\ldots, t$, $t\geq 0$. Moreover, updates \eqref{eq:prop_invariance_update1} and \eqref{eq:prop_invariance_update2} satisfy
    \begin{eqnarray}
        \hw_t = V\hw^V_t, \quad Vg_t = g^V_t, \quad f \left(\hw_t \right) = f^V \left(\hw^V_t \right) \label{eq:invariance_func}
    \end{eqnarray}
    for all $t \geq 0$ when $\hw_0 = \hw^V_0 = 0 \in \R^d$. Furthermore we have 
\begin{eqnarray}\label{eq:invariance_grad}
         \left\| g_t^V \right\|_{V^{-2}}^2 &=& \left\| g_t\right\|^2.    
    \end{eqnarray}
\end{proposition}
The Proposition \ref{prop:invariance} highlights that the update rule of the form \eqref{eq:invariant_adagrad} satisfies a scale-invariance property for GLMs. In contrast, \algname{AdaGrad} does not satisfy \eqref{eq:invariance_func} and \eqref{eq:invariance_grad}. In Appendix \ref{sec:scale_invariance}, we illustrate numerically the scale-invariance of \algname{KATE} and the lack of the scale-invariance of \algname{AdaGrad}. We also emphasize that \algname{AdaGrad} with $\Delta = 0$ is known to be a scale-free method\footnote{The algorithm is called scale-free if for any $c > 0$, it generates the same sequence of points for functions $f$ and $cf$ given the same initialization and hyperparameters.  To the best of our knowledge, this definition is introduced by \citet{cesa2005improved, cesa2007improved} in the context of learning with expert advice and extended to the context of generic online convex optimization by \citet{orabona2015scale, orabona2018scale}. We emphasize that scale-freeness and scale-invariance are completely different concepts.}.

\paragraph{Design of \algname{KATE}.}In order to construct an algorithm following the update rule \eqref{eq:invariant_adagrad}, one may choose $m_t[k] = 1\; \forall k \in [d]$. However, the step size from \eqref{eq:invariant_adagrad} in this case may decrease very fast, and the resulting method does not necessarily converge. Therefore, we need a more aggressive choice of $m_t$, which grows with $t$. It motivates the construction of our algorithm \algname{\textcolor{PineGreen}{KATE}} (Algorithm \ref{alg:KATE}),\footnote{Note that, for $m_t = b_t \forall t$ we get the \algname{AdaGrad} algorithm.}  where we choose $m_t[k] = \sqrt{\eta [k] b_t^2[k] + \sum_{\tau = 0}^t \frac{g_{\tau}^2[k]}{b_{\tau}^2[k]}}$. Note that the term $\sum_{\tau = 0}^t \frac{g_{\tau}^2[k]}{b_{\tau}^2[k]}$ is scale-invariant for GLMs (follows from Proposition \ref{prop:invariance}). To make $m_t$ scale-invariant, we choose $\eta \in \R^d$ in the following way:
\begin{itemize}
    \item \underline{$\eta \to 0$:} When $\eta$ is very small, $m_t$ is also approximately scale-invariant for GLMs. 
    \item \underline{$\eta = \nicefrac{1}{\left(\nabla f(w_0)\right)^2}$:} In this case $\eta b_t^2 = \nicefrac{b_t^2}{\left( \nabla f(w_0)\right)^2}$ is scale-invariant for GLMs (follows from Proposition~\ref{prop:invariance}) as well as $m_t$.  
\end{itemize}
\begin{algorithm}[t]
    \caption{\algname{\textcolor{PineGreen}{KATE}}}
    \label{alg:KATE}
    \begin{algorithmic}[1]
        \REQUIRE Initial point $w_0 \in \R^d$, step size $\beta> 0, \eta \in \R^d_{+}$ and $b_{-1}, m_{-1} = 0$.
        \FOR{$t = 0, 1,...,T$}
        \STATE Compute $g_t \in \R^d$ such that $\Exp{g_t} = \nabla f(w_t)$.
        \STATE $b_t^2 = b_{t-1}^2 + g_t^2$
        \STATE \colorbox{green!20}{$m_t^2 = m_{t-1}^2 + \eta g_t^2 + \frac{g_t^2}{b_t^2}$}
        \STATE $w_{t+1} = w_t - \frac{\beta m_t}{b_t^2}g_t$
        \ENDFOR
    \end{algorithmic}
\end{algorithm}
\algname{\textcolor{PineGreen}{KATE}} can be rewritten in the following coordinate form
    \begin{equation}\label{eq:algo_coordinate}
        w_{t+1}[k] = w_t[k] - \nu_t[k] g_t[k], \qquad \forall k \in [d],
    \end{equation}
    where $g_t$ is an unbiased estimator of $\nabla f(w_t)$ and the per-coefficient step size $\nu_t[k]$ is defined as 
    \begin{eqnarray}\label{eq:KATEstep}
        \nu_t[k] \coloneqq \frac{\beta \sqrt{\eta[k] b^2_t[k] + \sum_{\tau = 0}^t \frac{g^2_{\tau}[k]}{b_{\tau}^2[k]}}}{b_t^2[k]}.
    \end{eqnarray}
Note that the numerator of the steps $\nu_t[k]$ is increasing with iterations $t$. However, one of the crucial properties of this step size choice is that the steps always decrease with $t$, which we rely on in our convergence analysis. 
\begin{lemma}[Decreasing step size]\label{lemma:decreasing_step}
For $\nu_t[k]$ defined in \eqref{eq:algo_coordinate} we have
    \begin{equation}
        \nu_{t+1}[k] \leq \nu_t[k], \qquad \forall k \in [d]. \label{eq:decreasing_step}
    \end{equation}
\end{lemma}

\paragraph{Comparison with the scale-invariant version of \algname{AdaGrad} by \citet{orabona2015generalized}.} In the special case of GLMs, \citet{orabona2015generalized} propose a different version of \algname{AdaGrad}. The method is proposed for the case of online convex optimization, and in the case of standard optimization with GLMs \eqref{eq:GLM_opt_unscaled}, it has the following form
\begin{equation}
    w_0 \eqdef 0,\quad w_{t+1} \eqdef -\beta \frac{\sum_{\tau=0}^t \nabla f_{i_\tau}(w_{\tau})}{a_{t}^2 \sqrt{d}\sqrt{\gamma^2 + \sum_{\tau=0}^t \left(\nicefrac{\nabla f_{i_\tau}(w_{\tau})}{a_\tau}\right)^2}},\quad a_t \eqdef \max\limits_{\tau=0,\ldots,t}|x_{i_{\tau}}|, \label{eq:scale_invariant_AdaGrad_OCC}
\end{equation}
where $\{i_\tau\}_{\tau=0}^t$ are arbitrary indices from $[n]$ (e.g., selected uniformly at random), functions $f_i:\R^d \to \R$ are defined as $f_i(w)\eqdef \varphi_i(x_i^\top w)$ for $i\in [n]$, and $\gamma$ is such that $f_i(w)$ is $\gamma$-Lipschitz for $i \in [n]$. In this setup, the update rule of \algname{KATE} with $w_0 = 0$ can be written as follows:
\begin{gather*}
   w_{t+1} \eqdef -\beta \sum\limits_{\tau = 0}^t \frac{m_\tau}{b_\tau^2} \nabla f_{i_\tau}(w_\tau),\quad m_t\eqdef \sqrt{\eta \sum_{\tau=0}^{t}(\nabla f_{i_\tau}(w_\tau))^2 + \sum_{\tau=0}^t \nicefrac{(\nabla f_{i_\tau}(w_\tau))^2}{b_\tau^2}}, 
\end{gather*}
where $b_t\eqdef \sqrt{\sum_{\tau=0}^t (\nabla f_{i_\tau}(w_\tau))^2}$, $\{i_\tau\}_{\tau=0}^t$ are sampled from $[n]$ uniformly at random. Although both methods can be seen as variations of \algname{AdaGrad} due to the terms $\sum_{\tau=0}^t \left(\nicefrac{\nabla f_{i_\tau}(w_{\tau}}{a_\tau}\right)^2$ and $\sum_{\tau=0}^t \left(\nabla f_{i_\tau}(w_{\tau})\right)^2$ respectively, the scale-invariance is achieved quite differently in these methods. The method from \eqref{eq:scale_invariant_AdaGrad_OCC} uses the feature vectors explicitly in the update rule to ensure scale-invariance: indeed, the square root in the definition of $w_{t+1}$ is independent of scaling, and $a_t^2$ in the denominator ensures that $\hat w_{t+1} = V\hat{w}_{t+1}^V$ if we define them similarly to \algname{KATE} (see equations \eqref{eq:prop_invariance_update1}-\eqref{eq:prop_invariance_update2}). In contrast, \algname{KATE} achieves the scale-invariance by removing the square root from the denominator (as explained earlier). Moreover, unlike the method from \eqref{eq:scale_invariant_AdaGrad_OCC}, \algname{KATE} does not use the feature vectors explicitly in its update rule (only in the gradients of $f_{i_{\tau}}$) and, thus, can be used for general stochastic optimization (not necessarily for the case of GLMs).

\section{Convergence Analysis}\label{sec:analysis}
In this section, we present and discuss the convergence guarantees of 
\algname{\textcolor{PineGreen}{KATE}}. In the first subsection, we list the assumptions made about the problem.
\subsection{Assumptions}
In all our theoretical results, we assume that $f$ is smooth as defined below.
\begin{assumption}[$L$-smooth]\label{as:smooth}
    Function $f$ is $L$-smooth, i.e. for all $w, w' \in \R^d$
    \begin{equation}\label{eq:smooth}
        \hspace{-.15cm} f(w') \leq f(w) + \left \langle \nabla f(w), w' - w\right \rangle + \frac{L}{2} \left \| w - w'\right \|^2.
    \end{equation}
\end{assumption}
This assumption is standard in the literature of adaptive methods \citep{li2019convergence, ward2020adagrad, liu2022convergence,nguyen2018sgd,
nguyen2021inexact,
nguyen2017sarah,beznosikov2021random}. Moreover, we assume that at any iteration $t$ of \algname{\textcolor{PineGreen}{KATE}},
we can access $g_t$ --- a noisy and unbiased estimate of $\nabla f(w_t)$.
We also make the following assumption on the noise of the gradient estimate $g_t$.
\begin{assumption}[Bounded Variance] For fixed constant $\sigma > 0$, the variance of the stochastic gradient $g_t$ $\left(\text{unbiased estimate of } \nabla f(w_t) \right)$ at any time $t$ satisfies
\begin{equation}\label{eq:BV}
    \Expt{ \| g_t - \nabla f(w_t)  \|^2} \leq \sigma^2.  \tag{\textcolor{BrickRed}{\algname{BV}}}
\end{equation}    
\end{assumption}
Bounded variance is a common assumption to study the convergence of stochastic gradient-based methods. Several assumptions on stochastic gradients are used in the literature to explore the adaptive methods. \citet{ward2020adagrad} used the \ref{eq:BV}, while \citet{liu2022convergence} assumed the sub-Weibull noise, i.e. $\Exp{\exp{\left( \nicefrac{\left\|g_t - \nabla f(w_t) \right\|}{\sigma}\right)^{\nicefrac{1}{\theta}}}} \leq \exp{(1)}$ for some $\theta > 0$, to prove the convergence of \algname{\textcolor{PineGreen}{AdaGradNorm}}. \citet{li2019convergence} assumes sub-Gaussian ($\theta = \nicefrac{1}{2}$ in sub-Weibull condition) noise to study a variant of \algname{\textcolor{PineGreen}{AdaGrad}}. However, sub-Gaussian noise is strictly stronger than \ref{eq:BV}. Recently, \citet{faw2022power} analyzed \algname{\textcolor{PineGreen}{AdaGradNorm}} under a more relaxed condition known as affine variance $\left(\text{i.e. }\Expt{\left\| g_t - \nabla f(w_t)\right\|^2} \leq \sigma_0^2 + \sigma_1^2 \left\| \nabla f(w_t)\right\|^2 \right)$. 

\subsection{Main Results}
In this section, we cover the main convergence guarantees of \algname{\textcolor{PineGreen}{KATE}} for both deterministic and stochastic setups.
\paragraph{Deterministic setting.} We first present our results for the deterministic setting. In this setting, we consider the gradient estimate to have no noise (i.e. $\sigma^2 = 0$) and $g_t = \nabla f(w_t)$. The main result in this setting is summarized below. 

\begin{theorem}\label{theorem:deterministic_nonconvex}
    Suppose $f$ satisfy Assumption \ref{as:smooth} and $g_t = \nabla f(w_t)$. Moreover, $\beta > 0$ and $\eta[k] > 0$ are chosen such that $\nu_0[k] \leq \frac{1}{L}$ for all $k \in [d]$. Then the iterates of \algname{\textcolor{PineGreen}{KATE}} satisfies
    \begin{eqnarray*}
    \min_{t\leq T} \left\| \nabla f(w_t) \right\|^2 & \leq & \frac{\left(\tfrac{2(f(w_0) - f_*)}{\sqrt{\eta_0} \beta} + \sum_{k = 1}^d b_0[k]\right)^2}{T+1}, 
    \end{eqnarray*}
    where $\eta_0 \coloneqq \min_{k \in [d]} \eta[k]$.
\end{theorem}
\paragraph{Discussion on Theorem \ref{theorem:deterministic_nonconvex}.} Theorem \ref{theorem:deterministic_nonconvex} establishes an $\mathcal{O}\left(\nicefrac{1}{T}\right)$ convergence rate for \algname{\textcolor{PineGreen}{KATE}}, which is optimal for finding a first-order stationary point of a non-convex problem \citep{carmon2020lower}. However, this result is not parameter-free. To prove the convergence, we assume that $\nu_0[k] \leq \frac{1}{L},\; \forall k \in [d]$ in Theorem \ref{theorem:deterministic_nonconvex}, which is equivalent to $\beta \sqrt{1 + \eta_0 \left(\nabla_k f(w_0) \right)^2} \leq \nicefrac{\left(\nabla_k f(w_0) \right)^2}{L},\; \forall k \in [d]$. Note that the later condition holds for sufficiently small (dependent on $L$) values of $\beta, \eta_0 > 0$.

However, it is possible to derive a parameter-free version of Theorem \ref{theorem:deterministic_nonconvex}. Indeed, Lemma \ref{lemma:decreasing_step} implies that the step sizes are decreasing. Therefore, we can break down the analysis of \algname{\textcolor{PineGreen}{KATE}} into two phases: Phase I when $\nu_0[k] > \nicefrac{1}{L}$ and Phase II when $\nu_0[k] \leq \nicefrac{1}{L}$, when the current analysis works, and then follow the proof techniques of \citet{ward2020adagrad} and \citet{xie2020linear}. We leave this extension as a possible future direction of our work.

\paragraph{Stochastic setting.} Next, we present the convergence guarantees for \algname{KATE} in the stochastic case, when we can access an unbiased gradient estimate $g_t$ with non-zero noise. 
\begin{theorem}\label{theorem:stoch_nonconvex}
    Suppose $f$ satisfy Assumption \ref{as:smooth} and $g_t$ is an unbiased estimator of $\nabla f(w_t)$ such that \ref{eq:BV} holds. Moreover, we assume $\| \nabla f(w_t)\|^2 \leq \gamma^2$ for all $t$. Then the iterates of \algname{\textcolor{PineGreen}{KATE}} satisfy
    \begin{eqnarray*}
        \min_{t \leq T} \Exp{\left\| \nabla f(w_t) \right\|} \leq \left( \frac{\|g_0\|}{T} + \frac{2(\gamma + \sigma)}{\sqrt{T}}\right)^{\nicefrac{1}{2}} \sqrt{\frac{2 \mathcal{C}_f}{\beta \sqrt{\eta_0}}},
    \end{eqnarray*}
    where $\eta_0 \coloneqq \min_{k \in [d]} \eta[k]$ and 
    \begin{eqnarray*}
     \mathcal{C}_f & \coloneqq & f(w_0) - f_* + 2\beta \sigma \textstyle{\sum}_{k = 1}^d \sqrt{\eta[k]} \log \left( \tfrac{e(\sigma^2 + \gamma^2) T}{g_0^2[k]}\right)\\
        \hspace{-.75cm} && + 
    \textstyle{\sum}_{k = 1}^d \left(\tfrac{\beta^2 \eta[k] L}{2} + \tfrac{\beta^2 L}{2 g_0^2[k]} \right) \log \left( \tfrac{e(\sigma^2 + \gamma^2) T}{g_0^2[k]}\right).
    \end{eqnarray*}
    \normalsize
\end{theorem}

\paragraph{Comparison with prior work.} Theorem \ref{theorem:stoch_nonconvex} shows an $\mathcal{O}  ( \nicefrac{\log^{\nicefrac{1}{2}} T}{T^{\nicefrac{1}{4}}})$ convergence rate for \algname{\textcolor{PineGreen}{KATE}} with respect to the metric $ \min_{t \leq T} \Exp{\left\| \nabla f(w_t) \right\|}$ for the stochastic setting. Note that, in the stochastic setting, \algname{\textcolor{PineGreen}{KATE}} achieves a slower rate than Theorem \ref{theorem:deterministic_nonconvex} due to noise accumulation. Up to the logarithmic factor, this rate is optimal \citep{arjevani2023lower}. Similar rates for the same metric follow from the results\footnote{\citet{defossez2020simple} derive $\cO(\nicefrac{\log T}{\sqrt{T}})$ convergence rates for \algname{AdaGrad} and \algname{Adam} in terms of $\min_{t \leq T}\Exp{\left\| \nabla f(w_t)\right\|^2}$ which is not smaller than $\min_{t \leq T}\left(\Exp{\left\| \nabla f(w_t)\right\|}\right)^2$.} of \citep{defossez2020simple} for \algname{AdaGrad} and \algname{Adam}.

Finally, \citet{li2019convergence} considers a variant of \algname{AdaGrad} closely related to \algname{KATE}:
\begin{eqnarray}\label{eq:AdaGrad_past}
    w_{t+1} = w_t 
    - \frac{\beta g_t}
            {
             \left(
            \diag
                  \left(
                  \Delta I + \textstyle{\sum}_{\tau = 1}^{t-1} g_{\tau} g_{\tau}^{\top} 
                 \right) 
             \right)^{\tfrac{1}{2}+\varepsilon}
             },
\end{eqnarray}
for some $\varepsilon \in [0,\nicefrac{1}{2})$ and $\Delta > 0$.  It differs from \algname{AdaGrad} in two key aspects: the denominator of the stepsize does not contain the last stochastic gradient, and also, instead of the square root of the sum of squared gradients, this sum is taken in the power of $\nicefrac{1}{2} + \varepsilon$. However, the results from \citet{li2019convergence} do not imply convergence for the case of $\varepsilon = \nicefrac{1}{2}$, which is expected since, in this case, the stepsize converges to zero too quickly in general. To compensate for such a rapid decrease, in \algname{KATE}, we introduce an increasing sequence $m_t$ in the numerator of the stepsize. 

\paragraph{Proof technique.} Compared to the \algname{\textcolor{PineGreen}{AdaGrad}}, \algname{\textcolor{PineGreen}{KATE}} uses more aggressive steps (the larger numerator of \algname{\textcolor{PineGreen}{KATE}} due to the extra term $\sum_{\tau = 0}^t \nicefrac{g^2_{\tau}[k]}{b^2_{\tau}[k]}$). Therefore, we expect \algname{\textcolor{PineGreen}{KATE}} to have better empirical performance. However, introducing $\sum_{\tau = 0}^t \nicefrac{g^2_{\tau}[k]}{b^2_{\tau}[k]}$ in the numerator raises additional technical difficulties in the proof technique. Fortunately, as we rigorously show, the \algname{\textcolor{PineGreen}{KATE}} steps $\nu_t[k]$ retain some of the critical properties of \algname{\textcolor{PineGreen}{AdaGrad}} steps. For instance, they (i) are lower bounded by \algname{\textcolor{PineGreen}{AdaGrad}} steps up to a constant, (ii) decrease with iteration $t$ (Lemma \ref{lemma:decreasing_step}), and (iii) have closed-form upper bounds for $\sum_{t = 0}^T \nu^2_t[k] g^2_t[k]$. These are indeed the primary building blocks of our proof technique.

\section{Numerical Experiments}\label{sec:numerical_results} 
\vspace{-.35cm}
In this section, we implement \algname{\textcolor{PineGreen}{KATE}} in several machine learning tasks to evaluate its performance. To ensure transparency and facilitate reproducibility, we provide an access to the source code for all of our experiments at \url{https://github.com/nazya/KATE}.
\vspace{-.25cm}
\subsection{Logistic Regression}\label{section:logistic_reg}
In this section, we consider the logistic regression model
\begin{eqnarray}\label{eq:logistic_reg}
    \min_{w \in \R^d} f(w) = \frac{1}{n}\sum_{i = 1}^n  \log \left( 1 + \exp \left( -y_i x_i^{\top}w \right)\right),
\end{eqnarray}
to elaborate on the scale-invariance and robustness of \algname{\textcolor{PineGreen}{KATE}} for various initializations. For the experiments of this Section \ref{section:logistic_reg}, we used Mac mini (M1, 2020), RAM 8 GB and storage 256 GB. Each of these plots took about 20 minutes to run.

\subsubsection{Robustness of \algname{KATE}}\label{sec:robustness}
 To conduct this experiment, we set the total number of samples to $1000$ (i.e. $n=1000$). Here, we simulate the independent vectors $x_i \in \R^{20}$ such that each entry is from $\mathcal{N}(0, 1)$. Moreover, we generate a diagonal matrix $V \in \R^{20 \times 20}$ such that $\log V_{kk} \overset{\mathrm{iid}}{\sim} \text{Unif}(-10, 10), \ \forall k \in [20]$. Similarly, we generate $w^* \in \R^{20}$ with each component from $\mathcal{N}(0,1)$ and set the labels 
\begin{eqnarray*}
    y_i = 
    \left\{
    \begin{matrix} 
      1, & \quad x_i^{\top}Vw^* \geq 0, \\
      -1, &\quad   x_i^{\top}Vw^* < 0, 
   \end{matrix}\right. \qquad \forall i \in [n].
   \end{eqnarray*}
We compare \algname{\textcolor{PineGreen}{KATE}}'s performance with four other algorithms: \algname{\textcolor{PineGreen}{AdaGrad}}, \algname{\textcolor{PineGreen}{AdaGradNorm}}, \algname{\textcolor{PineGreen}{SGD-decay}} and \algname{\textcolor{PineGreen}{SGD-constant}}, similar to the section 5.1 of \citet{ward2020adagrad}. For each algorithm, we initialize with $w_0 = 0 \in \R^{20}$ and independently draw a sample of mini-batch size $10$ to update the weight vector $w_t$. We compare the algorithms $\bullet$~\algname{\textcolor{PineGreen}{AdaGrad}} with stepsize $\tfrac{\beta}{\sqrt{\Delta + \sum_{\tau = 0}^t g^2_{\tau}}}$, $\bullet$ \algname{\textcolor{PineGreen}{AdaGradNorm}} with step size $\frac{\beta}{\sqrt{\Delta + \sum_{\tau = 0}^t \left\|g_{\tau} \right\|^2}}$, $\bullet$~\algname{\textcolor{PineGreen}{SGD-decay}} with stepsize $\nicefrac{\beta}{\Delta \sqrt{t+1}}$, and $\bullet$~\algname{\textcolor{PineGreen}{SGD-constant}} with step size $\nicefrac{\beta}{\Delta}$. Similarly, for \algname{KATE} we use stepsize $\frac{\beta m_t}{b_t^2}$ where $m_t^2 = \eta b_t^2 + \sum_{\tau = 0}^t \nicefrac{g_{\tau}^2}{b_{\tau}^2}$ and $b_t^2 = \Delta + \sum_{\tau = 0}^t g_{\tau}^2$. Here, we choose $\beta = f(w_0) - f(w^*)$ and vary $\Delta$ in $\{10^{-8}, 10^{-6}, 10^{-4}, 10^{-2}, 1, 10^{2}, 10^{4}, 10^{6}, 10^{8} \}$. 

In Figures \ref{fig:initial_log10000}, \ref{fig:initial_log50000}, and \ref{fig:initial_log100000}, we plot the functional value $f(w_t)$ (on the $y$-axis) after $10^4, 5 \times 10^4$, and $10^5$ iterations, respectively. In theory, the convergence of \algname{\textcolor{PineGreen}{SGD}} requires the knowledge of smoothness constant $L$. Therefore, when the $\Delta$ is small (hence the stepsize is large), \algname{\textcolor{PineGreen}{SGD-decay}} and \algname{\textcolor{PineGreen}{SGD-constant}} diverge. However, the adaptive algorithms \algname{\textcolor{PineGreen}{KATE}}, \algname{\textcolor{PineGreen}{AdaGrad}}, and \algname{\textcolor{PineGreen}{AdaGradNorm}} can auto-tune themselves and converge for a wide range of $\Delta$s (even when the $\Delta$ is too small).  As we observe in Figure \ref{fig:initial_log},  when the $\Delta$ is small, \algname{\textcolor{PineGreen}{KATE}} outperforms all other algorithms. For instance, when $\Delta = 10^{-8}$, \algname{\textcolor{PineGreen}{KATE}} achieves a functional value of $10^{-3}$ after only $10^4$ iterations (see Figure \ref{fig:initial_log10000}), while other algorithms fail to achieve this even after $10^5$ iterations (see Figure \ref{fig:initial_log100000}). Furthermore, \algname{\textcolor{PineGreen}{KATE}} performs as well as \algname{\textcolor{PineGreen}{AdaGrad}} and better than other algorithms when the $\Delta$ is large. \textit{In particular, this experiment highlights that \algname{\textcolor{PineGreen}{KATE}} is robust to initialization $\Delta$.}

\begin{figure}[t]
\centering
\begin{subfigure}[b]{.32\textwidth}
    \centering
    \includegraphics[width=\textwidth]{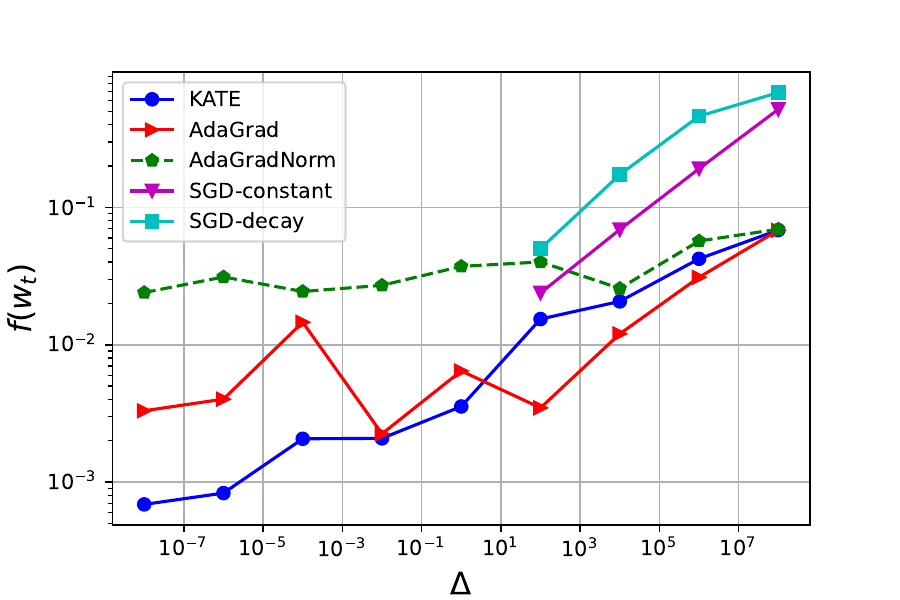}
    \caption{ Plot of $f(w_t)$}\label{fig:initial_log10000}
\end{subfigure}
\begin{subfigure}[b]{0.32\textwidth}
    \centering
    \includegraphics[width=\textwidth]{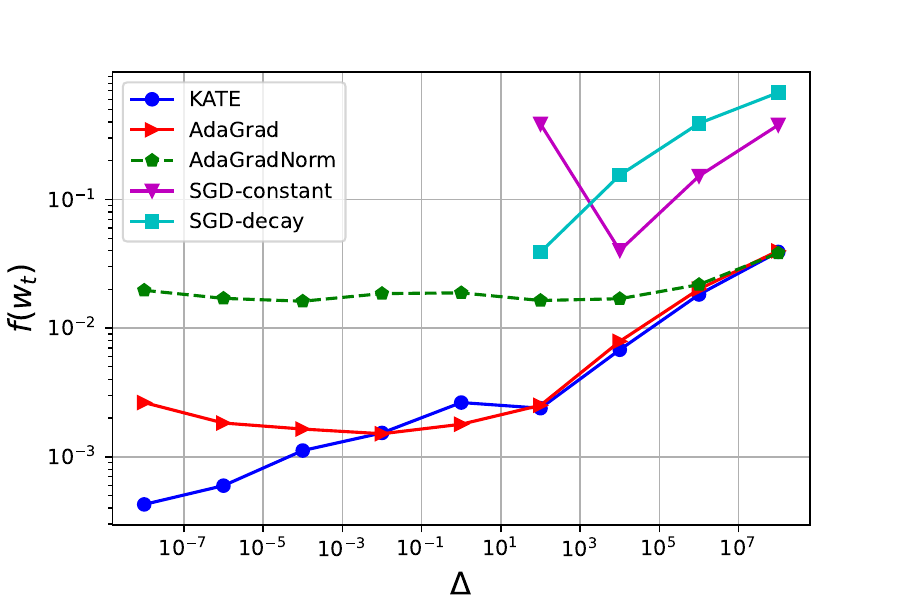}
    \caption{ Plot of Accuracy}\label{fig:initial_log50000}
\end{subfigure}
\begin{subfigure}[b]{0.32\textwidth}
    \centering
    \includegraphics[width=\textwidth]{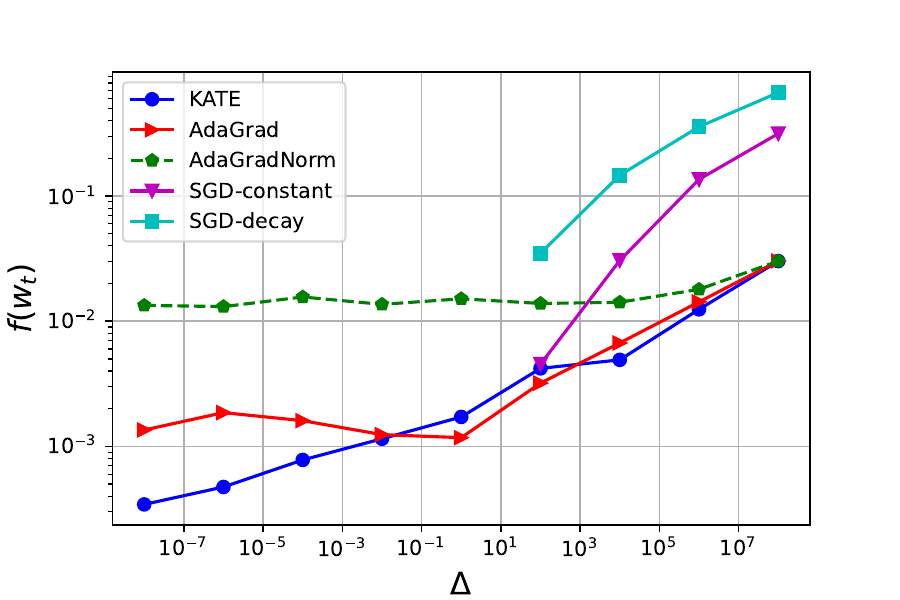}
    \caption{Plot of Gradient Norm}\label{fig:initial_log100000}
\end{subfigure}
    \caption{Comparison of \algname{\textcolor{PineGreen}{KATE}} with \algname{\textcolor{PineGreen}{AdaGrad}}, \algname{\textcolor{PineGreen}{AdaGradNorm}}, \algname{\textcolor{PineGreen}{SGD-decay}} and \algname{\textcolor{PineGreen}{SGD-constant}} for different values of $\Delta$ (on $x$-axis for logistic regression model. Figure \ref{fig:initial_log10000}, \ref{fig:initial_log50000} and \ref{fig:initial_log100000} plots the functional value $f(w_t)$ (on $y$-axis) after $10^4, 5 \times 10^4$, and $10^5$ iterations respectively.}\label{fig:initial_log}
\begin{subfigure}[b]{.32\textwidth}
    \centering
    \includegraphics[width=\textwidth]{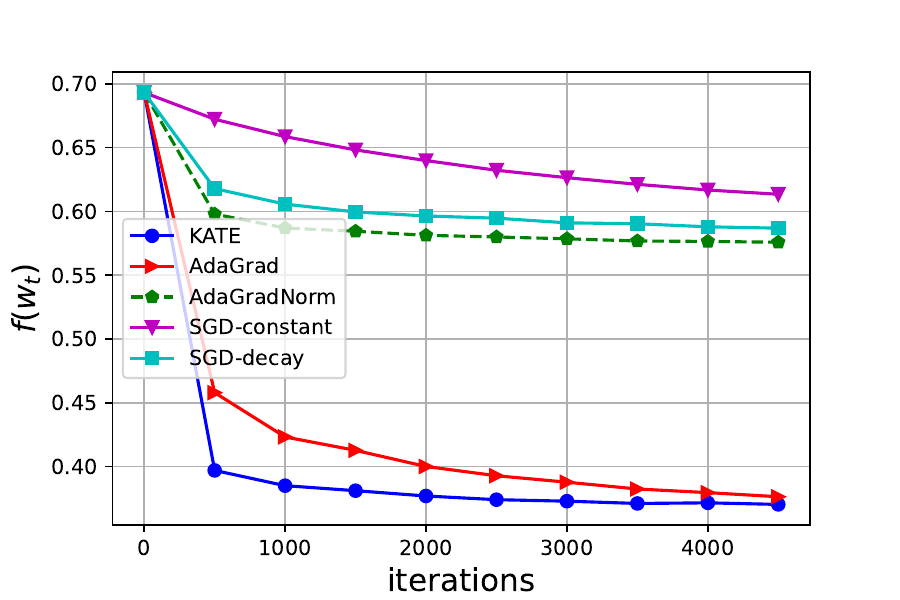}
    \caption{Dataset: heart}\label{fig:heart_fval}
\end{subfigure}
\begin{subfigure}[b]{0.32\textwidth}
    \centering
    \includegraphics[width=\textwidth]{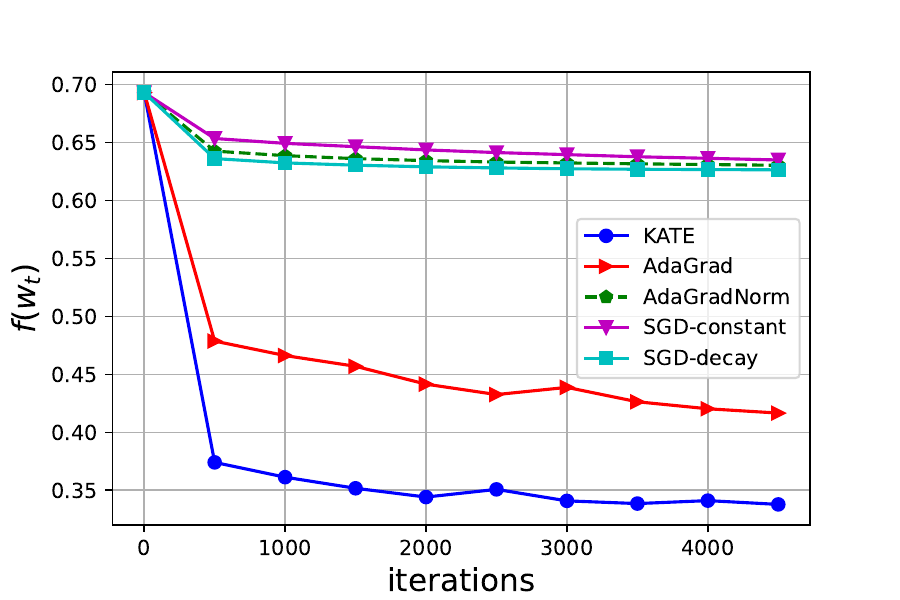}
    \caption{Dataset: australian}\label{fig:australian_fval}
\end{subfigure}
\begin{subfigure}[b]{0.32\textwidth}
    \centering
    \includegraphics[width=\textwidth]{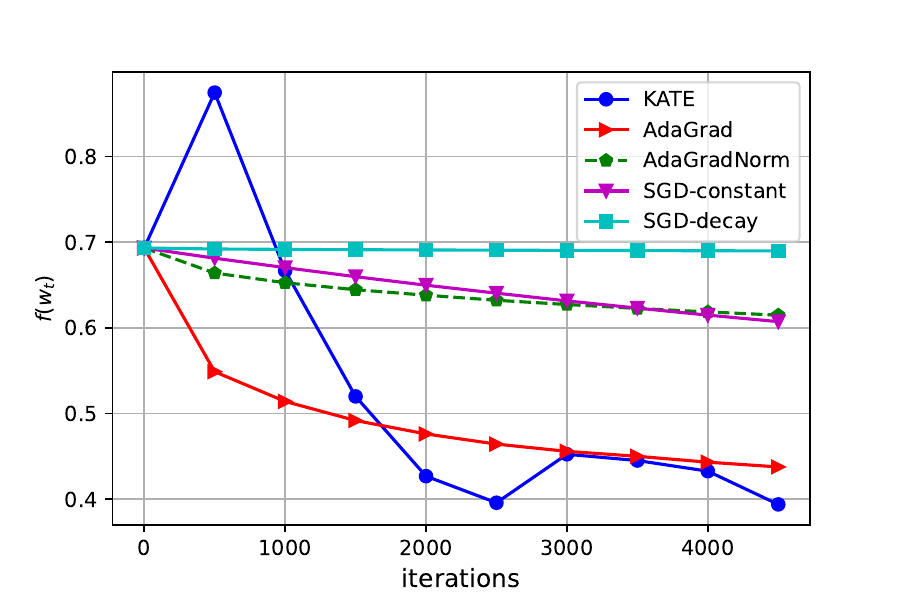}
    \caption{Dataset: splice}\label{fig:splice_fval}
\end{subfigure}
\begin{subfigure}[b]{.32\textwidth}
    \centering
    \includegraphics[width=\textwidth]{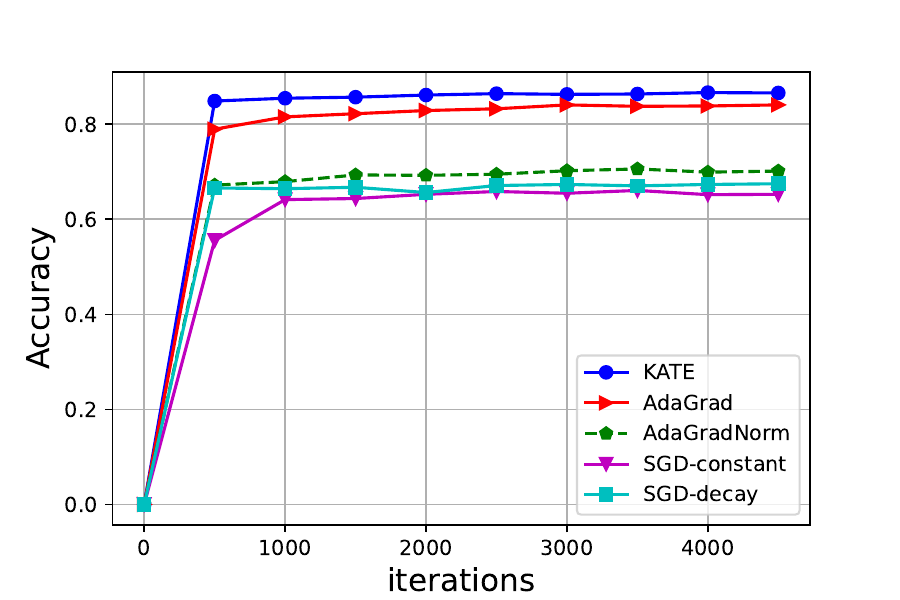}
    \caption{Dataset: heart}\label{fig:heart_accuracy}
\end{subfigure}
\begin{subfigure}[b]{0.32\textwidth}
    \centering
    \includegraphics[width=\textwidth]{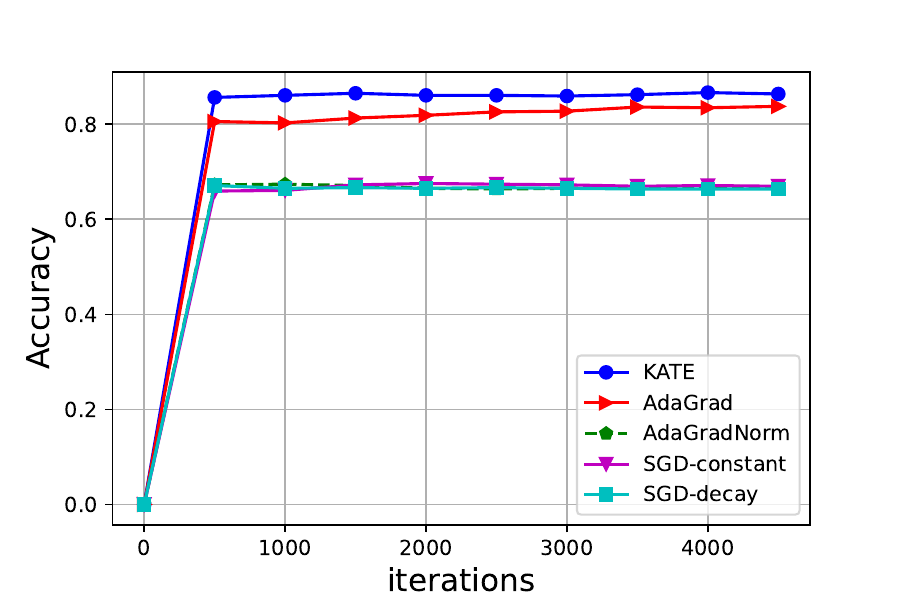}
    \caption{Dataset: australian}\label{fig:australian_accuracy}
\end{subfigure}
\begin{subfigure}[b]{0.32\textwidth}
    \centering
    \includegraphics[width=\textwidth]{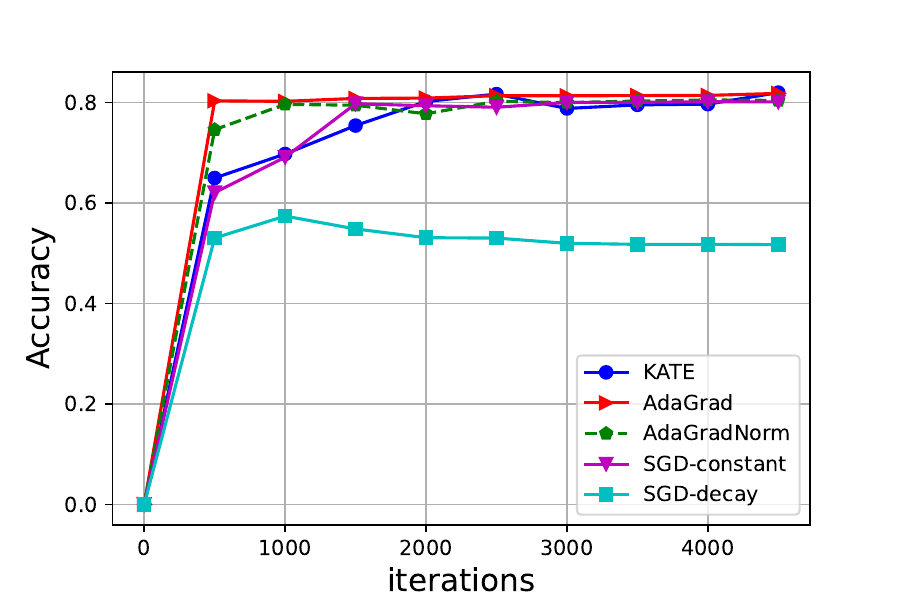}
    \caption{Dataset: splice}\label{fig:splice_accuracy}
\end{subfigure}
      \caption{Comparison of \algname{\textcolor{PineGreen}{KATE}} with \algname{\textcolor{PineGreen}{AdaGrad}}, \algname{\textcolor{PineGreen}{AdaGradNorm}}, \algname{\textcolor{PineGreen}{SGD-decay}} and \algname{\textcolor{PineGreen}{SGD-constant}} on datasets heart, australian, and splice from \href{https://www.csie.ntu.edu.tw/~cjlin/libsvmtools/datasets/binary.html}{LIBSVM}. Figures \ref{fig:heart_fval}, \ref{fig:australian_fval} and \ref{fig:splice_fval} plot the functional value $f(w_t)$, while \ref{fig:heart_accuracy}, \ref{fig:australian_accuracy} and \ref{fig:splice_accuracy} plot the accuracy on $y$-axis for $5,000$ iterations.}\label{fig:real_data}
      \vspace{-0.5cm}
\end{figure}

\subsubsection{Peformance of \algname{\textcolor{PineGreen}{KATE}} on Real Data} 
In this section, we examine \algname{KATE}'s performance on real data. We test \algname{KATE} on three datasets: heart, australian, and splice from the \href{https://www.csie.ntu.edu.tw/~cjlin/libsvmtools/datasets/binary.html}{LIBSVM} library \citep{chang2011libsvm}. The response variables $y_i$ of each of these datasets contain two classes, and we use them for binary classification tasks using a logistic regression model~\eqref{eq:logistic_reg}. We take $\eta = \nicefrac{1}{\left( \nabla f(w_0)\right)^2}$ for \algname{KATE} and tune $\beta$ in all the experiments. For tuning $\beta$, we do a grid search on the list $\{10^{-10}, 10^{-8}, 10^{-6}, 10^{-4}, 10^{-2}, 1 \}$. Similarly, we tune stepsizes for other algorithms. We take $5$ trials for each of these algorithms and plot the mean of their trajectories. 

We plot the functional value $f(w_t)$ (i.e. loss function) in Figures \ref{fig:heart_fval}, \ref{fig:australian_fval} and \ref{fig:splice_fval}, whereas Figures \ref{fig:heart_accuracy}, \ref{fig:australian_accuracy} and \ref{fig:splice_accuracy} plot the corresponding accuracy of the weight vector $w_t$ on the $y$-axis for $5,000$ iterations. We observe that \algname{KATE} performs superior to all other algorithms, even on real datasets.

\begin{figure*}[h]
   \centering
    \begin{minipage}[htp]{0.28\textwidth}
        \centering
        \includegraphics[width=1\linewidth]{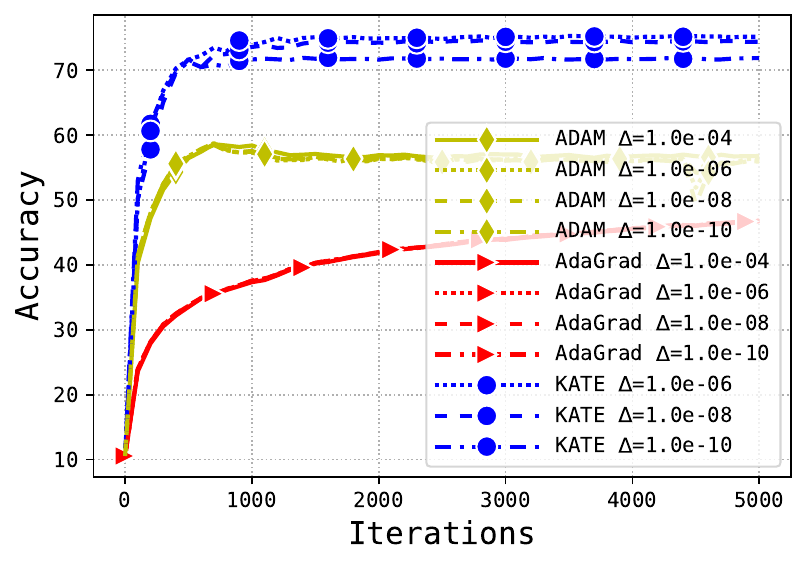}
        \includegraphics[width=1\linewidth]{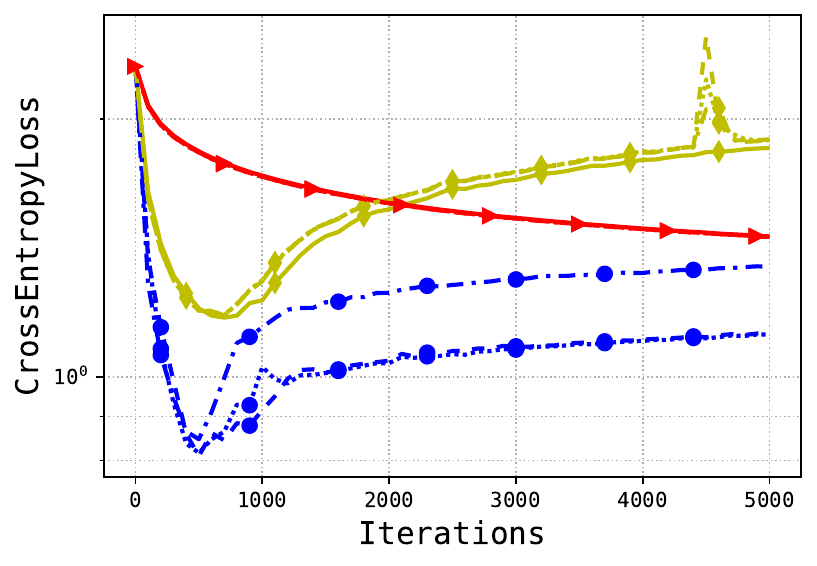}
        \caption{CIFAR10: $\eta = 0$}
        \label{fig:resnet-0}
    \end{minipage}
\hfill
    \begin{minipage}[htp]{0.28\textwidth}
        \centering
        \includegraphics[width=1\linewidth]{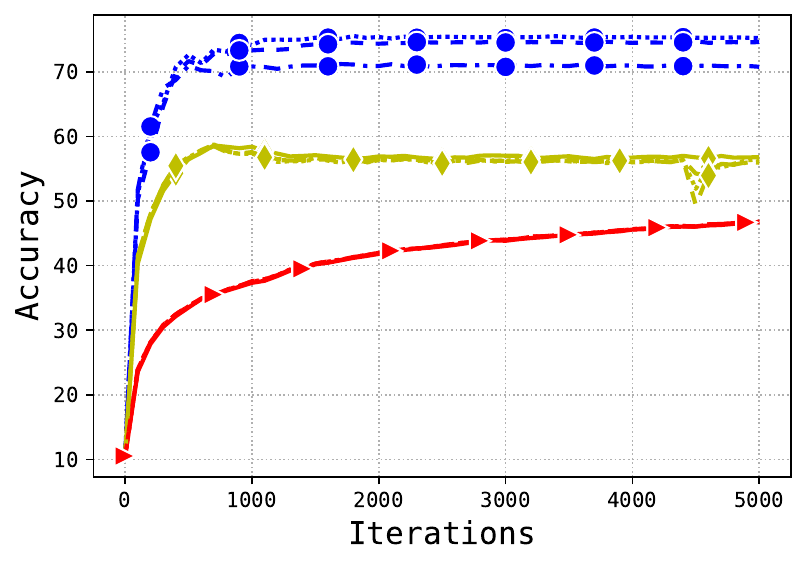}
        \includegraphics[width=1\linewidth]{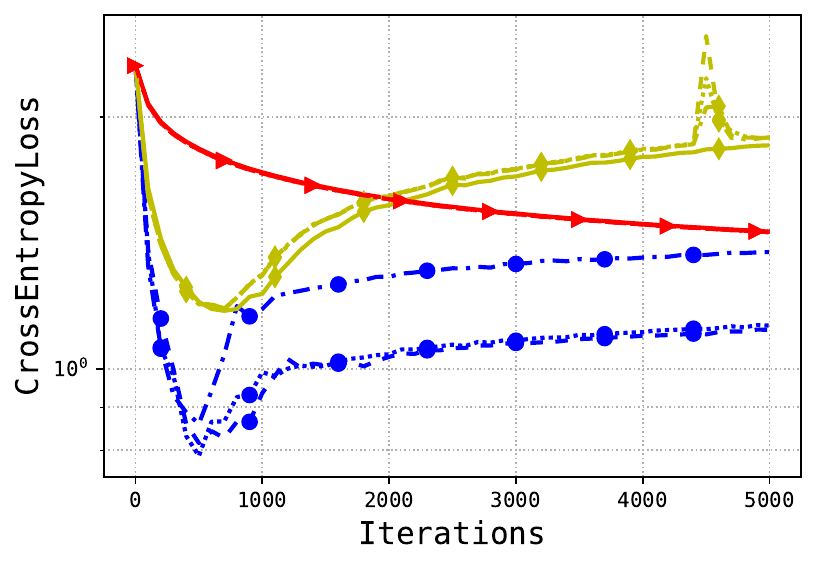}
        \caption{CIFAR10: $\eta = 0.001$}
        \label{fig:resnet-0.001}
    \end{minipage}
\hfill
    \begin{minipage}[htp]{0.28\textwidth}
        \centering
        \includegraphics[width=1\linewidth]{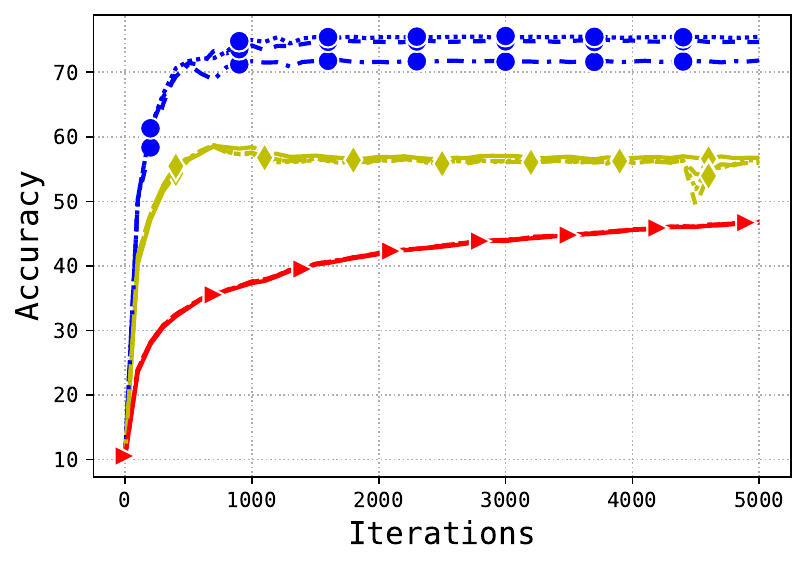}
        \includegraphics[width=1\linewidth]{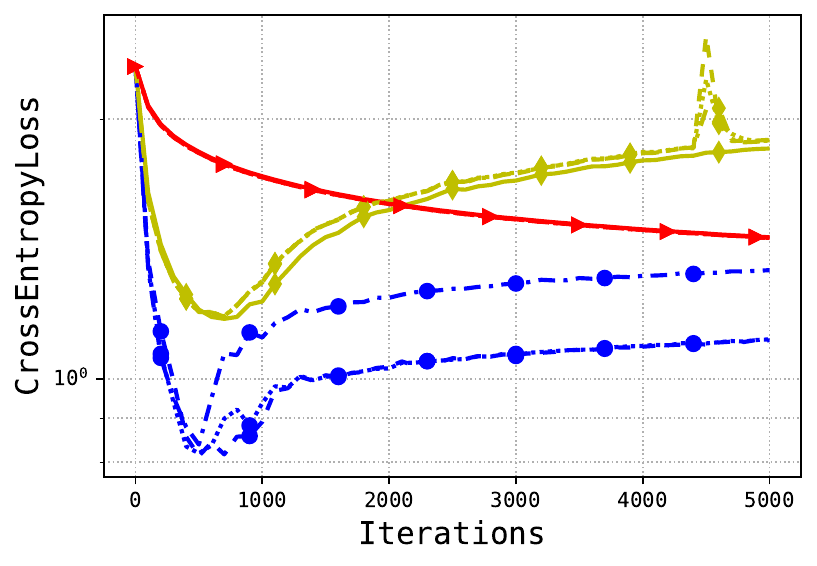}
        \caption{CIFAR10: $\eta = 0.1$}
        \label{fig:resnet-0.1}
    \end{minipage}
    \vspace{-0.3cm}
\end{figure*}

\begin{figure*}[htp]
   \centering
    \begin{minipage}[htp]{0.28\textwidth}
        \centering
        \includegraphics[width=1\linewidth]{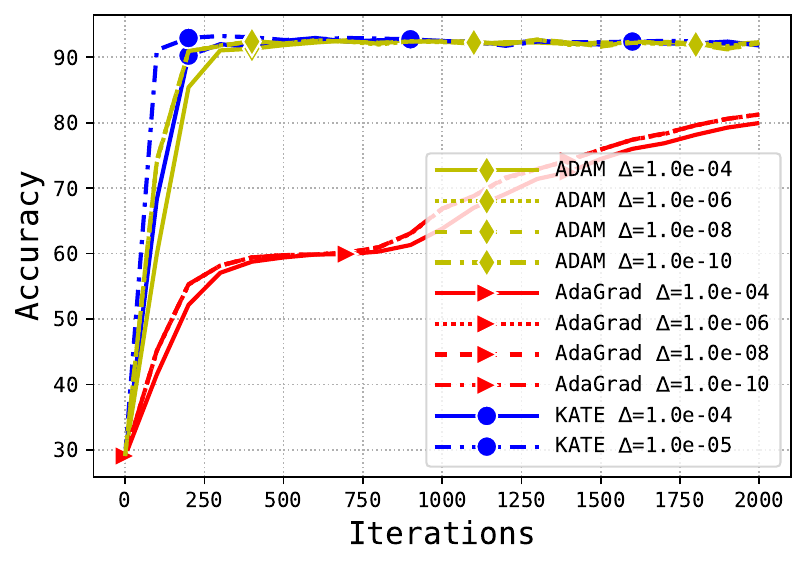}
        \includegraphics[width=1\linewidth]{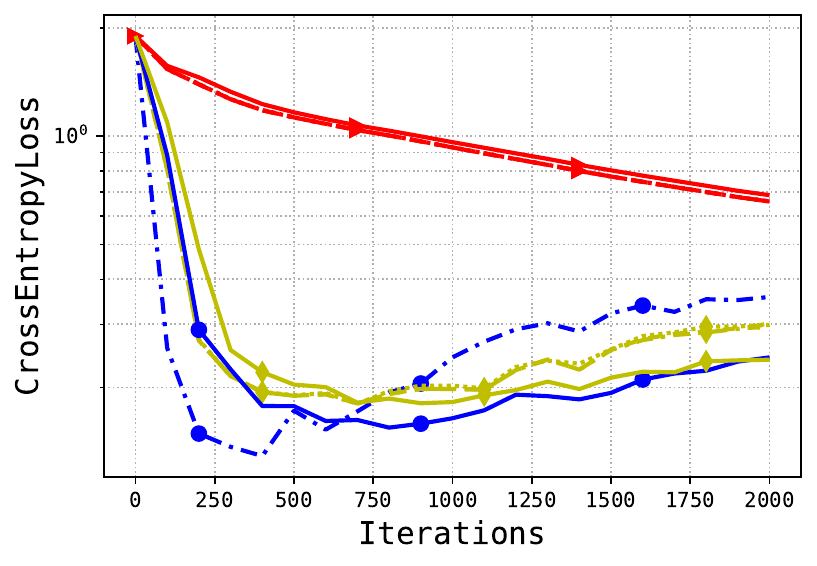}
        \caption{Emotion: $\eta=0$}
        \label{fig:bert-0}
    \end{minipage}
\hfill
    \begin{minipage}[htp]{0.28\textwidth}
        \centering
        \includegraphics[width=1\linewidth]{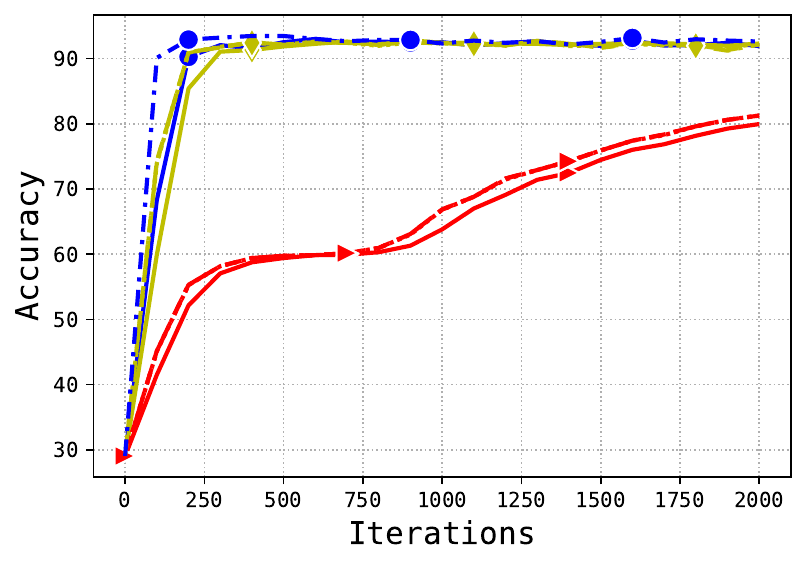}
        \includegraphics[width=1\linewidth]{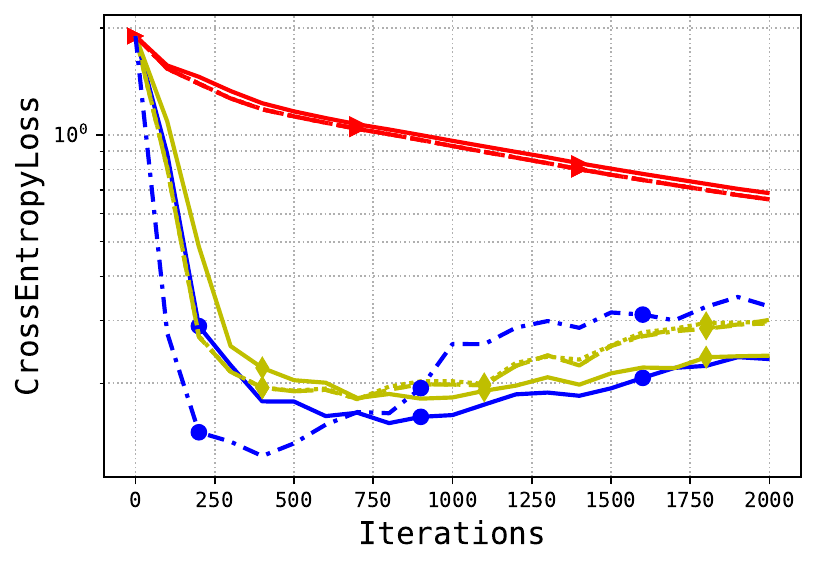}
        \caption{Emotion: $\eta = 0.001$}
        \label{fig:bert-0.001}
    \end{minipage}
\hfill
    \begin{minipage}[htp]{0.28\textwidth}
        \centering
        \includegraphics[width=1\linewidth]{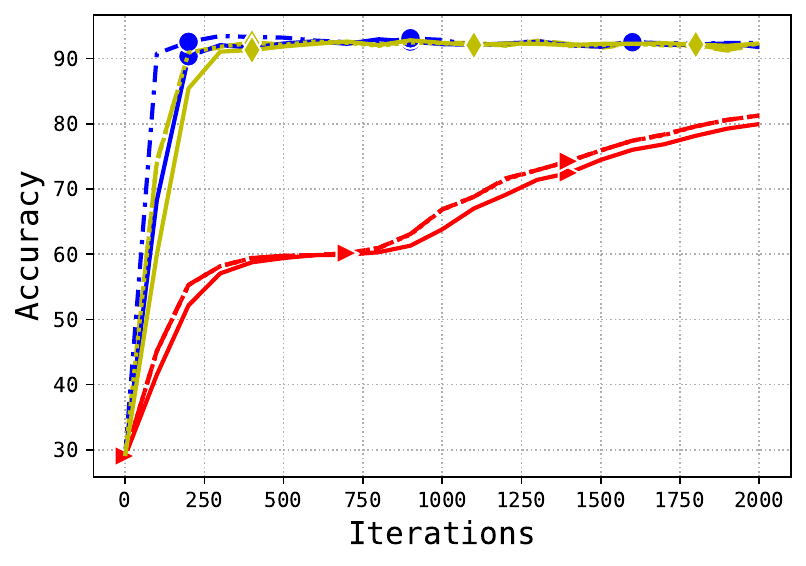}
        \includegraphics[width=1\linewidth]{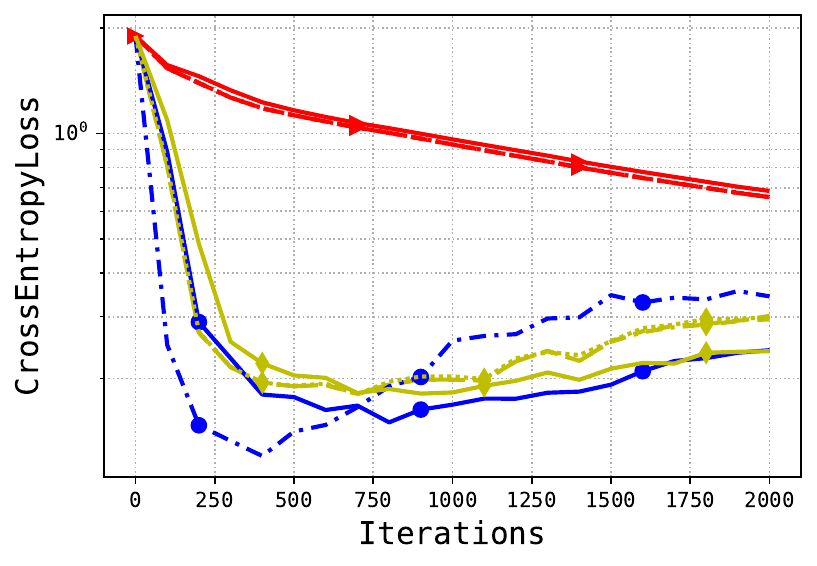}
        \caption{Emotion: $\eta = 0.1$}
        \label{fig:bert-0.1}
    \end{minipage}
    \vspace{-.7cm}
\end{figure*}
\vspace{-.25cm}
\subsection{Training of Neural Networks}\label{section:neuralnet}

In this section, we compare the performance of \algname{KATE}, \algname{AdaGrad} and \algname{Adam} on two tasks, i.e. training ResNet18~\citep{he2016deep} on the CIFAR10 dataset~\citep{krizhevsky2009learning} and BERT~\citep{devlin2018bert} fine-tuning on the emotions dataset~\citep{saravia-etal-2018-carer} from the Hugging Face Hub.
We use internal cluster with the following hardware: AMD EPYC 7552 48-Core Processor, 512GiB RAM, NVIDIA A100 40GB GPU, 200gb user storage space.
\vspace{-.35cm}
\paragraph{General comparison.} We choose standard parameters for \algname{Adam} ($\beta_1 = 0.9$ and $\beta_2 = 0.999$) that are default values in PyTorch and select the learning rate of $10^{-5}$ for all considered methods. We run \algname{KATE} with different values of $\eta \in \{0, 10^{-1}, 10^{-2}\}$. For the image classification task, we normalize the images (similar to \citet{horvath2020better}) and use a mini-batch size of 500. For the BERT fine-tuning, we use a mini-batch size 160 for all methods.

Figures~\ref{fig:resnet-0}-\ref{fig:bert-0.1} report the evolution of top-1 accuracy and cross-entropy loss (on the $y$-axis) calculated on the test data. For the image classification task, we observe that \algname{KATE} with different choices of $\eta$ outperforms \algname{Adam} and \algname{AdaGrad}. Finally, we also observe that \algname{KATE} performs comparably to \algname{Adam} on the BERT fine-tuning task and is better than \algname{AdaGrad}. These preliminary results highlight the potential of \algname{KATE} to be applied for training neural networks for different tasks. For BERT each run takes about 35 minutes, and 25 minutes for ResNet.

\begin{wrapfigure}{r}{0.3\textwidth}
    \centering
    \vspace{-6mm}
    \includegraphics[width=\linewidth]{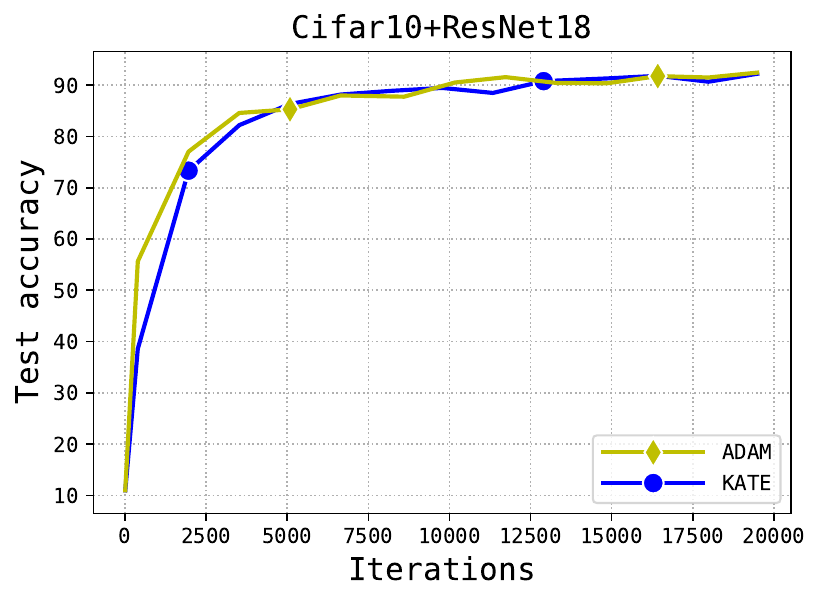}
    \caption{\small Cifar10: $\eta = 0.001$}
    \label{fig:bert-tuned}
    \includegraphics[width=\linewidth]{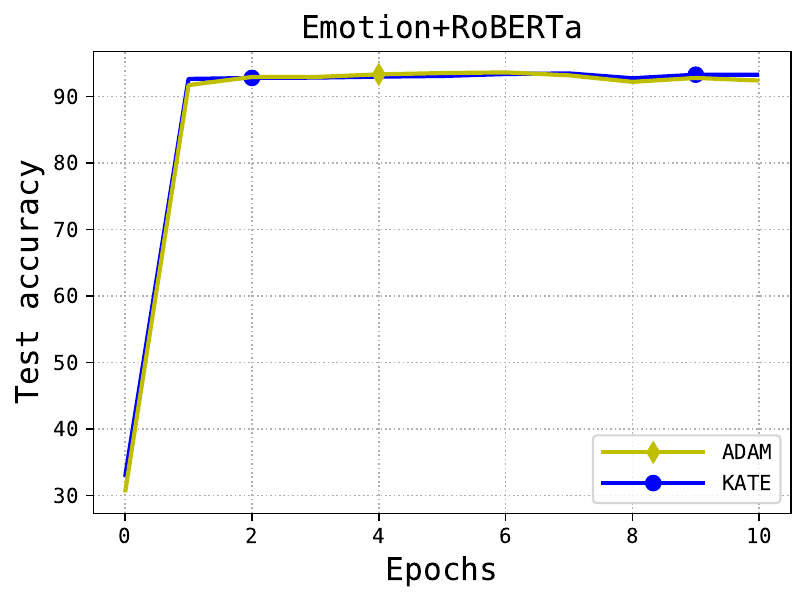}
    \caption{\small Emotion: $\eta = 0.001$}
    \label{fig:resnet-tuned}
    \vspace{-5mm}
\end{wrapfigure}

\paragraph{Hyper-parameters tuning.} Next, we compare baselines presented in~\cite{saravia-etal-2018-carer} for emotions classification and \cite{zhang2019lookahead} for image classification. These papers provide efficient setups for learning rates and learning rate schedulers that are reasonable to compare with. \cite{saravia-etal-2018-carer} performs a search of efficient learning rate and uses a linear learning rate scheduler with warmup for \algname{Adam} optimizer. A different learning rate (1e-5), $\Delta$=1e-5 and the same scheduler applied for \algname{KATE} lead to the same performance, see Figure \ref{fig:bert-tuned}. We would like to point out that it is challenging to find a reference for hyper-parameters for a certain setup. Thus, to fairly compare with~\cite{saravia-etal-2018-carer} we use distilroberta-base model. \cite{zhang2019lookahead} did a grid search for an efficient learning rate and used a multi-step scheduler for \algname{Adam} optimizer, decaying the learning rate by a factor of 5 at the 60th, 120th, and 160th epochs.~\cite{zhang2019lookahead} refers to~\cite{devries2017improved} for the code implementing special techniques, namely data augmentation and cutout to achieve higher accuracy. A different learning rate (1e-3), the same scheduler and $\Delta$=1e-3 applied for \algname{KATE} demonstrates comparable performance, see Figure \ref{fig:resnet-tuned}. For BERT each run takes about 20 minutes, while 100 minutes for ResNet.


\begin{ack}
We thank Francesco Orabona and Dmitry Kamzolov for the pointers to the related works that we missed while preparing the first version of this paper. We also thank anonymous reviewers for their useful feedback and suggestions.
\end{ack}

\small{\bibliography{ref}}

\begin{thebibliography}{}

\bibitem[Abdukhakimov et~al., 2023]{abdukhakimov2023sania}
Abdukhakimov, F., Xiang, C., Kamzolov, D., Gower, R., and Tak{\'a}{\v{c}}, M. (2023).
\newblock Sania: Polyak-type optimization framework leads to scale invariant stochastic algorithms.
\newblock {\em arXiv preprint arXiv:2312.17369}.

\bibitem[Abdukhakimov et~al., 2024]{abdukhakimov2023stochastic}
Abdukhakimov, F., Xiang, C., Kamzolov, D., and Tak{\'a}{\v{c}}, M. (2024).
\newblock Stochastic gradient descent with preconditioned polyak step-size.
\newblock {\em Computational Mathematics and Mathematical Physics}, 64(4):621--634.

\bibitem[Agresti, 2015]{agresti2015foundations}
Agresti, A. (2015).
\newblock {\em Foundations of linear and generalized linear models}.
\newblock John Wiley \& Sons.

\bibitem[Arjevani et~al., 2023]{arjevani2023lower}
Arjevani, Y., Carmon, Y., Duchi, J.~C., Foster, D.~J., Srebro, N., and Woodworth, B. (2023).
\newblock Lower bounds for non-convex stochastic optimization.
\newblock {\em Mathematical Programming}, 199(1-2):165--214.

\bibitem[Beznosikov and Tak{\'a}{\v{c}}, 2021]{beznosikov2021random}
Beznosikov, A. and Tak{\'a}{\v{c}}, M. (2021).
\newblock Random-reshuffled sarah does not need a full gradient computations.
\newblock In {\em Optimization for Machine Learning Workshop @ NeurIPS 2021}.

\bibitem[Carmon et~al., 2020]{carmon2020lower}
Carmon, Y., Duchi, J.~C., Hinder, O., and Sidford, A. (2020).
\newblock Lower bounds for finding stationary points i.
\newblock {\em Mathematical Programming}, 184(1-2):71--120.

\bibitem[Cesa-Bianchi et~al., 2005]{cesa2005improved}
Cesa-Bianchi, N., Mansour, Y., and Stoltz, G. (2005).
\newblock Improved second-order bounds for prediction with expert advice.
\newblock In {\em International Conference on Computational Learning Theory}, pages 217--232. Springer.

\bibitem[Cesa-Bianchi et~al., 2007]{cesa2007improved}
Cesa-Bianchi, N., Mansour, Y., and Stoltz, G. (2007).
\newblock Improved second-order bounds for prediction with expert advice.
\newblock {\em Machine Learning}, 66:321--352.

\bibitem[Chang and Lin, 2011]{chang2011libsvm}
Chang, C.-C. and Lin, C.-J. (2011).
\newblock Libsvm: a library for support vector machines.
\newblock {\em ACM transactions on intelligent systems and technology (TIST)}, 2(3):1--27.

\bibitem[Chezhegov et~al., 2024]{chezhegov2024local}
Chezhegov, S., Skorik, S., Khachaturov, N., Shalagin, D., Avetisyan, A., Beznosikov, A., Tak{\'a}{\v{c}}, M., Kholodov, Y., and Gasnikov, A. (2024).
\newblock Local methods with adaptivity via scaling.
\newblock {\em arXiv preprint arXiv:2406.00846}.

\bibitem[d'Aspremont et~al., 2018]{d2018optimal}
d'Aspremont, A., Guzman, C., and Jaggi, M. (2018).
\newblock Optimal affine-invariant smooth minimization algorithms.
\newblock {\em SIAM Journal on Optimization}, 28(3):2384--2405.

\bibitem[Defazio and Mishchenko, 2023]{defazio2023learning}
Defazio, A. and Mishchenko, K. (2023).
\newblock Learning-rate-free learning by {D}-adaptation.
\newblock {\em arXiv preprint arXiv:2301.07733}.

\bibitem[D{\'e}fossez et~al., 2020]{defossez2020simple}
D{\'e}fossez, A., Bottou, L., Bach, F., and Usunier, N. (2020).
\newblock A simple convergence proof of adam and adagrad.
\newblock {\em arXiv preprint arXiv:2003.02395}.

\bibitem[Devlin et~al., 2018]{devlin2018bert}
Devlin, J., Chang, M.-W., Lee, K., and Toutanova, K. (2018).
\newblock {BERT}: Pre-training of deep bidirectional transformers for language understanding.
\newblock {\em arXiv preprint arXiv:1810.04805}.

\bibitem[DeVries and Taylor, 2017]{devries2017improved}
DeVries, T. and Taylor, G.~W. (2017).
\newblock Improved regularization of convolutional neural networks with cutout.
\newblock {\em arXiv preprint arXiv:1708.04552}.

\bibitem[D'Orazio et~al., 2021]{d2021stochastic}
D'Orazio, R., Loizou, N., Laradji, I., and Mitliagkas, I. (2021).
\newblock Stochastic mirror descent: Convergence analysis and adaptive variants via the mirror stochastic polyak stepsize.
\newblock {\em arXiv preprint arXiv:2110.15412}.

\bibitem[Duchi et~al., 2011]{duchi2011adaptive}
Duchi, J., Hazan, E., and Singer, Y. (2011).
\newblock Adaptive subgradient methods for online learning and stochastic optimization.
\newblock {\em Journal of machine learning research}, 12(7).

\bibitem[Faw et~al., 2022]{faw2022power}
Faw, M., Tziotis, I., Caramanis, C., Mokhtari, A., Shakkottai, S., and Ward, R. (2022).
\newblock The power of adaptivity in sgd: Self-tuning step sizes with unbounded gradients and affine variance.
\newblock In {\em Conference on Learning Theory}, pages 313--355. PMLR.

\bibitem[Frome, 1983]{frome1983analysis}
Frome, E.~L. (1983).
\newblock The analysis of rates using poisson regression models.
\newblock {\em Biometrics}, pages 665--674.

\bibitem[Gower et~al., 2021]{gower2021stochastic}
Gower, R.~M., Defazio, A., and Rabbat, M. (2021).
\newblock Stochastic polyak stepsize with a moving target.
\newblock {\em arXiv preprint arXiv:2106.11851}.

\bibitem[He et~al., 2016]{he2016deep}
He, K., Zhang, X., Ren, S., and Sun, J. (2016).
\newblock Deep residual learning for image recognition.
\newblock In {\em Proceedings of the IEEE conference on computer vision and pattern recognition}, pages 770--778.

\bibitem[He et~al., 2018]{he2018dual}
He, X., Tappenden, R., and Takac, M. (2018).
\newblock Dual free adaptive minibatch sdca for empirical risk minimization.
\newblock {\em Frontiers in Applied Mathematics and Statistics}, 4:33.

\bibitem[Horv\'{a}th and Richt\'{a}rik, 2020]{horvath2020better}
Horv\'{a}th, S. and Richt\'{a}rik, P. (2020).
\newblock A better alternative to error feedback for communication-efficient distributed learning.
\newblock {\em arXiv preprint arXiv:2006.11077}.

\bibitem[Hosmer~Jr et~al., 2013]{hosmer2013applied}
Hosmer~Jr, D.~W., Lemeshow, S., and Sturdivant, R.~X. (2013).
\newblock {\em Applied logistic regression}, volume 398.
\newblock John Wiley \& Sons.

\bibitem[Kingma and Ba, 2014]{kingma2014adam}
Kingma, D.~P. and Ba, J. (2014).
\newblock Adam: A method for stochastic optimization.
\newblock {\em arXiv preprint arXiv:1412.6980}.

\bibitem[Krizhevsky and Hinton, 2009]{krizhevsky2009learning}
Krizhevsky, A. and Hinton, G. (2009).
\newblock Learning multiple layers of features from tiny images.

\bibitem[Li et~al., 2023]{li2022sp2}
Li, S., Swartworth, W.~J., Tak{\'a}{\v{c}}, M., Needell, D., and Gower, R.~M. (2023).
\newblock Sp2: A second order stochastic polyak method.
\newblock {\em ICLR}.

\bibitem[Li and Orabona, 2019]{li2019convergence}
Li, X. and Orabona, F. (2019).
\newblock On the convergence of stochastic gradient descent with adaptive stepsizes.
\newblock In {\em The 22nd international conference on artificial intelligence and statistics}, pages 983--992. PMLR.

\bibitem[Liu et~al., 2022]{liu2022convergence}
Liu, Z., Nguyen, T.~D., Ene, A., and Nguyen, H.~L. (2022).
\newblock On the convergence of adagrad on $\mathbb{R}^d$: Beyond convexity, non-asymptotic rate and acceleration.
\newblock {\em arXiv preprint arXiv:2209.14827}.

\bibitem[Loizou et~al., 2021]{loizou2021stochastic}
Loizou, N., Vaswani, S., Laradji, I.~H., and Lacoste-Julien, S. (2021).
\newblock Stochastic polyak step-size for sgd: An adaptive learning rate for fast convergence.
\newblock In {\em International Conference on Artificial Intelligence and Statistics}, pages 1306--1314. PMLR.

\bibitem[McMahan and Streeter, 2010]{mcmahan2010adaptive}
McMahan, H.~B. and Streeter, M. (2010).
\newblock Adaptive bound optimization for online convex optimization.
\newblock {\em arXiv preprint arXiv:1002.4908}.

\bibitem[Mishchenko and Defazio, 2023]{mishchenko2023prodigy}
Mishchenko, K. and Defazio, A. (2023).
\newblock Prodigy: An expeditiously adaptive parameter-free learner.
\newblock {\em arXiv preprint arXiv:2306.06101}.

\bibitem[Nelder and Wedderburn, 1972]{nelder1972generalized}
Nelder, J.~A. and Wedderburn, R.~W. (1972).
\newblock Generalized linear models.
\newblock {\em Journal of the Royal Statistical Society Series A: Statistics in Society}, 135(3):370--384.

\bibitem[Nesterov, 2018]{nesterov2018lectures}
Nesterov, Y. (2018).
\newblock {\em Lectures on convex optimization}, volume 137.
\newblock Springer.

\bibitem[Nesterov and Nemirovskii, 1994]{nesterov1994interior}
Nesterov, Y. and Nemirovskii, A. (1994).
\newblock {\em Interior-point polynomial algorithms in convex programming}.
\newblock SIAM.

\bibitem[Nguyen et~al., 2017a]{nguyen2017sarah}
Nguyen, L., Liu, J., Scheinberg, K., and Tak{\'a}{\v{c}}, M. (2017a).
\newblock Sarah: A novel method for machine learning problems using stochastic recursive gradient.
\newblock In {\em In 34th International Conference on Machine Learning, ICML 2017}.

\bibitem[Nguyen et~al., 2017b]{nguyen2017stochastic}
Nguyen, L.~M., Liu, J., Scheinberg, K., and Tak{\'a}{\v{c}}, M. (2017b).
\newblock Stochastic recursive gradient algorithm for nonconvex optimization.
\newblock {\em arXiv preprint arXiv:1705.07261}.

\bibitem[Nguyen et~al., 2018]{nguyen2018sgd}
Nguyen, L.~M., Nguyen, P.~H., van Dijk, M., Richt{\'a}rik, P., Scheinberg, K., and Tak{\'a}{\v{c}}, M. (2018).
\newblock Sgd and hogwild! convergence without the bounded gradients assumption.
\newblock In {\em In 34th International Conference on Machine Learning, ICML 2018}.

\bibitem[Nguyen et~al., 2021]{nguyen2021inexact}
Nguyen, L.~M., Scheinberg, K., and Tak{\'a}{\v{c}}, M. (2021).
\newblock Inexact sarah algorithm for stochastic optimization.
\newblock {\em Optimization Methods and Software}, 36(1):237--258.

\bibitem[Oberman and Prazeres, 2019]{oberman2019stochastic}
Oberman, A.~M. and Prazeres, M. (2019).
\newblock Stochastic gradient descent with polyak's learning rate.
\newblock {\em arXiv preprint arXiv:1903.08688}.

\bibitem[Orabona et~al., 2015]{orabona2015generalized}
Orabona, F., Crammer, K., and Cesa-Bianchi, N. (2015).
\newblock A generalized online mirror descent with applications to classification and regression.
\newblock {\em Machine Learning}, 99:411--435.

\bibitem[Orabona and P{\'a}l, 2015]{orabona2015scale}
Orabona, F. and P{\'a}l, D. (2015).
\newblock Scale-free algorithms for online linear optimization.
\newblock In {\em International Conference on Algorithmic Learning Theory}, pages 287--301. Springer.

\bibitem[Orabona and P{\'a}l, 2018]{orabona2018scale}
Orabona, F. and P{\'a}l, D. (2018).
\newblock Scale-free online learning.
\newblock {\em Theoretical Computer Science}, 716:50--69.

\bibitem[Orvieto et~al., 2022]{orvieto2022dynamics}
Orvieto, A., Lacoste-Julien, S., and Loizou, N. (2022).
\newblock Dynamics of sgd with stochastic polyak stepsizes: Truly adaptive variants and convergence to exact solution.
\newblock {\em Advances in Neural Information Processing Systems}, 35:26943--26954.

\bibitem[Polyak, 1969]{polyak1969minimization}
Polyak, B.~T. (1969).
\newblock Minimization of unsmooth functionals.
\newblock {\em USSR Computational Mathematics and Mathematical Physics}, 9(3):14--29.

\bibitem[Reddi et~al., 2019]{reddi2019convergence}
Reddi, S.~J., Kale, S., and Kumar, S. (2019).
\newblock On the convergence of adam and beyond.
\newblock {\em arXiv preprint arXiv:1904.09237}.

\bibitem[Robbins and Monro, 1951]{robbins1951stochastic}
Robbins, H. and Monro, S. (1951).
\newblock A stochastic approximation method.
\newblock {\em The annals of mathematical statistics}, pages 400--407.

\bibitem[Saravia et~al., 2018]{saravia-etal-2018-carer}
Saravia, E., Liu, H.-C.~T., Huang, Y.-H., Wu, J., and Chen, Y.-S. (2018).
\newblock {CARER}: Contextualized affect representations for emotion recognition.
\newblock In {\em Proceedings of the 2018 Conference on Empirical Methods in Natural Language Processing}, pages 3687--3697, Brussels, Belgium. Association for Computational Linguistics.

\bibitem[Shalev-Shwartz and Ben-David, 2014]{shalev2014understanding}
Shalev-Shwartz, S. and Ben-David, S. (2014).
\newblock {\em Understanding machine learning: From theory to algorithms}.
\newblock Cambridge university press.

\bibitem[Shi et~al., 2023]{shi2023aisarah}
Shi, Z., Sadiev, A., Loizou, N., Richt{\'a}rik, P., and Tak{\'a}{\v{c}}, M. (2023).
\newblock {AI}-{SARAH}: Adaptive and implicit stochastic recursive gradient methods.
\newblock {\em Transactions on Machine Learning Research}.

\bibitem[Tak{\'a}{\v{c}} et~al., 2013]{takavc2013mini}
Tak{\'a}{\v{c}}, M., Bijral, A., Richt{\'a}rik, P., and Srebro, N. (2013).
\newblock Mini-batch primal and dual methods for svms.
\newblock In {\em In 30th International Conference on Machine Learning, ICML 2013}.

\bibitem[Tieleman and Hinton, 2012]{tieleman2012rmsprop}
Tieleman, T. and Hinton, G. (2012).
\newblock Rmsprop: Divide the gradient by a running average of its recent magnitude. coursera: Neural networks for machine learning.
\newblock {\em COURSERA Neural Networks Mach. Learn}, 17.

\bibitem[Ward et~al., 2020]{ward2020adagrad}
Ward, R., Wu, X., and Bottou, L. (2020).
\newblock Adagrad stepsizes: Sharp convergence over nonconvex landscapes.
\newblock {\em The Journal of Machine Learning Research}, 21(1):9047--9076.

\bibitem[Xie et~al., 2020]{xie2020linear}
Xie, Y., Wu, X., and Ward, R. (2020).
\newblock Linear convergence of adaptive stochastic gradient descent.
\newblock In {\em International conference on artificial intelligence and statistics}, pages 1475--1485. PMLR.

\bibitem[Zhang et~al., 2019]{zhang2019lookahead}
Zhang, M., Lucas, J., Ba, J., and Hinton, G.~E. (2019).
\newblock Lookahead optimizer: k steps forward, 1 step back.
\newblock {\em Advances in neural information processing systems}, 32.

\end{thebibliography}


\newpage
\appendix
\part*{Supplementary Material}


\tableofcontents

\newpage
\section{Technical Lemmas}\label{section:technical_lemma}
\vspace{.5cm}
\begin{lemma}[AM-GM]\label{lemma:AM-GM} For $\lambda > 0$ we have
    \begin{eqnarray}\label{eq:AM-GM}
        ab \leq \frac{\lambda}{2}a^2 + \frac{1}{2 \lambda} b^2.
    \end{eqnarray}
\end{lemma}

\begin{lemma}[Cauchy-Schwarz Inequality]\label{lemma:CS}
    For $a_1, \cdots, a_n, b_1, \cdots, b_n \in \R$ we have
    \begin{eqnarray}\label{eq:CS}
        \left( \sum_{i = 1}^n a_i^2\right) \left( \sum_{i = 1}^n b_i^2\right) & \geq & \left( \sum_{i = 1}^n a_i b_i\right)^2.
    \end{eqnarray}
\end{lemma}

\begin{lemma}[Holder's Inequality]
    Suppose $X, Y$ are two random variables and $p, q > 1$ satisfy $\frac{1}{p} + \frac{1}{q} = 1$. Then 
    \begin{equation}\label{eq:holder}
        \E \left( |XY| \right) \leq \left( \E \left(|X|^p \right)\right)^{\frac{1}{p}} \left( \E \left( |Y|^q \right)\right)^{\frac{1}{q}}.
    \end{equation}
\end{lemma}

\begin{lemma}[Jensen's Inequality]\label{lemma:Jensen}
    For a convex function $g: \R^d \to \R$ and a random variable $X$ such that $\E{\left( \Psi(X)\right)}$ and $\Psi \left( \E{\left(X \right)}\right)$ are finite, we have 
    \begin{equation}\label{eq:Jensen}
        \Psi \left( \E{\left(X \right)}\right) \leq \E{\left( \Psi(X)\right)}.
    \end{equation}
\end{lemma}

\begin{lemma}\label{lemma:inequality}
For $a_1, a_2, \cdots, a_n \geq 0$ and $b_1, b_2, \cdots, b_n > 0$ we have
\begin{eqnarray}
    \sum_{i = 1}^n \frac{a_i}{\sqrt{b_i}} \geq \frac{\sum_{i =1}^n a_i}{\sqrt{\sum_{i = 1}^n b_i}}. \label{eq:inequality}
\end{eqnarray}
\end{lemma}

\begin{proof}
Expanding the LHS of \eqref{eq:inequality} we get
\begin{eqnarray}
    \left( \sum_{i = 1}^n \frac{a_i}{\sqrt{b_i}}\right)^2 & = & \sum_{i = 1}^n \frac{a_i^2}{b_i} + 2 \sum_{i \neq j} \frac{a_i a_j}{\sqrt{b_i b_j}} \notag\\
    & \geq & \sum_{i = 1}^n \frac{a_i^2}{b_i}. \label{eq:inequality_eq1}
\end{eqnarray}
The last inequality follows from $\frac{a_i}{\sqrt{b_i}} \geq 0$ for all $i \in [n]$. Now, using Cauchy-Schwarz Inequality \eqref{eq:CS}, we have
\begin{eqnarray}
    \left(\sum_{i = 1}^n \frac{a_i^2}{b_i} \right) \left( \sum_{i = 1}^n b_i \right) & \geq & \left( \sum_{i = 1}^n a_i \right)^2. \label{eq:inequality_eq2}
\end{eqnarray}
Then combining \eqref{eq:inequality_eq1} and \eqref{eq:inequality_eq2}, we get
\begin{eqnarray*}
    \left( \sum_{i = 1}^n \frac{a_i}{\sqrt{b_i}}\right)^2 \left( \sum_{i = 1}^n b_i \right) & \geq & \left( \sum_{i = 1}^n a_i \right)^2.
\end{eqnarray*}
Finally dividing both sides by $\sum_{i = 1}^n b_i$ and taking square root we get the desired result.
\end{proof}

\newpage

\begin{lemma}
For $k \in [d]$ and $t \geq 1$ we have
\begin{eqnarray}\label{eq:bound_surrogate}
    \Expt{\left( \frac{\beta \sqrt{\eta[k]}}{\sqrt{b_{t-1}^2[k] + \left(\nabla_k f(w_t) \right)^2 + \sigma^2}} - \nu_t[k] \right) \nabla_k f(w_t) g_t[k] }
    &\leq &  \frac{ \beta \sqrt{\eta[k]} \left( \nabla_k f(w_t) \right)^2}{2 \sqrt{b^2_{t-1}[k] + \left( \nabla f(w_t)\right)^2 + \sigma^2}} \notag \\
    && + 2\beta \sqrt{\eta[k]} \sigma \Expt{\frac{g^2_t[k]}{b_t^2[k]}}
\end{eqnarray}
\end{lemma}
\begin{proof}
    Note that, using $\nu_t[k] \geq \frac{\beta \sqrt{\eta[k]}}{b_t[k]}$ we have 
\begin{eqnarray}
    && \frac{\beta \sqrt{\eta[k]}}{\sqrt{b_{t-1}^2[k] + \left(\nabla_k f(w_t) \right)^2 + \sigma^2}} - \nu_t[k] \notag\\
    & \leq & \beta \sqrt{\eta[k]} \left( \frac{1}{\sqrt{b_{t-1}^2[k] + \left(\nabla_k f(w_t) \right)^2 + \sigma^2}} - \frac{1}{b_t[k]}\right) \notag\\
    & = & \beta \sqrt{\eta[k]} \left( \frac{b_t^2[k] - b_{t-1}^2[k] - (\nabla_k f(w_t))^2 - \sigma^2}{b_t[k] \sqrt{b_{t-1}^2[k] + \left(\nabla_k f(w_t) \right)^2 + \sigma^2} \left(b_t[k] + \sqrt{b_{t-1}^2[k] + \left(\nabla_k f(w_t) \right)^2 + \sigma^2} \right) } \right) \notag\\
    & = & \beta \sqrt{\eta[k]} \left( \frac{g_t^2[k] - (\nabla_k f(w_t))^2 - \sigma^2}{b_t[k] \sqrt{b_{t-1}^2[k] + \left(\nabla_k f(w_t) \right)^2 + \sigma^2} \left(b_t[k] + \sqrt{b_{t-1}^2[k] + \left(\nabla_k f(w_t) \right)^2 + \sigma^2} \right) } \right) \notag\\
    & = & \beta \sqrt{\eta[k]} \left( \frac{ \left( g_t[k] + \nabla_k f(w_t) \right) \left( g_t[k] - \nabla_k f(w_t) \right) - \sigma^2}{b_t[k] \sqrt{b_{t-1}^2[k] + \left(\nabla_k f(w_t) \right)^2 + \sigma^2} \left(b_t[k] + \sqrt{b_{t-1}^2[k] + \left(\nabla_k f(w_t) \right)^2 + \sigma^2} \right) } \right) \notag\\
    &\leq& \frac{ \beta \sqrt{\eta[k]}\left| \left( g_t[k] + \nabla_k f(w_t) \right) \left( g_t[k] - \nabla_k f(w_t) \right) \right|}{b_t[k] \sqrt{b_{t-1}^2[k] + \left(\nabla_k f(w_t) \right)^2 + \sigma^2} \left(b_t[k] + \sqrt{b_{t-1}^2[k] + \left(\nabla_k f(w_t) \right)^2 + \sigma^2} \right)} \notag\\
    && + \frac{ \beta \sqrt{\eta[k]} \sigma^2 }{b_t[k] \sqrt{b_{t-1}^2[k] + \left(\nabla_k f(w_t) \right)^2 + \sigma^2} \left(b_t[k] + \sqrt{b_{t-1}^2[k] + \left(\nabla_k f(w_t) \right)^2 + \sigma^2} \right)} \notag\\
    &\leq& \frac{ \beta \sqrt{\eta[k]}\left| g_t[k] - \nabla_k f(w_t) \right|}{b_t[k] \sqrt{b^2_{t-1}[k] + (\nabla_k f(w_t))^2 + \sigma^2}} + \frac{\beta \sqrt{\eta[k]} \sigma }{b_t[k] \sqrt{b^2_{t-1}[k] + (\nabla_k f(w_t))^2 + \sigma^2}}. \label{eq:bound_surrogate_eq1}
\end{eqnarray}
Note that the second last inequality follows from the use of triangle inequality in the following way
\begin{eqnarray*}
\left( g_t[k] + \nabla_k f(w_t) \right) \left( g_t[k] - \nabla_k f(w_t) \right) - \sigma^2 & \leq & \left|\left( g_t[k] + \nabla_k f(w_t) \right) \left( g_t[k] - \nabla_k f(w_t) \right) - \sigma^2 \right| \\
& \leq & \left| \left( g_t[k] + \nabla_k f(w_t) \right) \left( g_t[k] - \nabla_k f(w_t) \right) \right| + \sigma^2,
\end{eqnarray*}
while the last inequality follows from
\begin{eqnarray*}
b_t[k] + \sqrt{b_{t-1}^2[k] + \left(\nabla_k f(w_t) \right)^2 + \sigma^2} & \geq & \left|g_t[k] \right| + \left|\nabla_k f(w_t) \right| \geq \left| g_t[k] + \nabla_k f(w_t)\right|, \\
b_t[k] + \sqrt{b_{t-1}^2[k] + \left(\nabla_k f(w_t) \right)^2 + \sigma^2} & \geq & \sigma.    
\end{eqnarray*}
Then from \eqref{eq:bound_surrogate_eq1} we have
\begin{eqnarray}
    && \Expt{\left( \frac{\beta \sqrt{\eta[k]}}{\sqrt{b_{t-1}^2[k] + \left(\nabla_k f(w_t) \right)^2 + \sigma^2}} - \nu_t[k] \right) \nabla_k f(w_t) g_t[k] } \notag\\
    & \leq & \underbrace{\beta \sqrt{\eta[k]} \Expt{ \frac{\left| g_t[k] - \nabla_k f(w_t) \right| \left| \nabla_k f(w_t) \right| \left| g_t[k] \right|}{b_t[k] \sqrt{b^2_{t-1}[k] + (\nabla_k f(w_t))^2 + \sigma^2}}}}_\textrm{term I} \notag\\
    && + \underbrace{\beta \sqrt{\eta[k]} \Expt{\frac{\sigma \left| \nabla_k f(w_t) \right| \left| g_t[k] \right|}{b_t[k] \sqrt{b^2_{t-1}[k] + (\nabla_k f(w_t))^2 + \sigma^2}}}}_\textrm{term II}. \label{eq:bound_surrogate_eq2}
\end{eqnarray}
For term I in \eqref{eq:bound_surrogate_eq2}, we use Lemma \ref{lemma:AM-GM} with
\begin{eqnarray*}
\lambda &=& \frac{2 \sigma^2}{\sqrt{b^2_{t-1}[k] + \left( \nabla_k f(w_t)\right)^2 + \sigma^2}}, \\
a & = & \frac{\left| g_t[k]\right|}{b_t[k]}, \\
b &= & \frac{\left| g_t[k] - \nabla_k f(w_t) \right| \left| \nabla_k f(w_t) \right|}{\sqrt{b^2_{t-1}[k] + (\nabla_k f(w_t))^2 + \sigma^2}},
\end{eqnarray*}
to get 
\begin{eqnarray}
    && \beta \sqrt{\eta[k]} \Expt{ \frac{\left| g_t[k] - \nabla_k f(w_t) \right| \left| \nabla_k f(w_t) \right| \left| g_t[k] \right|}{b_t[k] \sqrt{b^2_{t-1}[k] + (\nabla_k f(w_t))^2 + \sigma^2}}} \notag\\
    & \leq & \frac{\beta \sqrt{\eta[k]} \sqrt{b^2_{t-1}[k] + \left( \nabla_k f(w_t)\right)^2 + \sigma^2}}{4 \sigma^2} \frac{\left( \nabla_k f(w_t)\right)^2 \Expt{g_t[k] - \nabla_k f(w_t)}^2}{b^2_{t-1}[k] + \left( \nabla_k f(w_t)\right)^2 + \sigma^2} \notag \\
    && + \frac{\beta \sqrt{\eta[k]} \sigma^2}{\sqrt{b^2_{t-1}[k] + \left( \nabla_k f(w_t)\right)^2 + \sigma^2}} \Expt{\frac{g^2_t[k]}{b^2_t[k]}} \notag\\
    & \leq & \frac{ \beta \sqrt{\eta[k]} \left( \nabla_k f(w_t)\right)^2}{4 \sqrt{b_{t-1}^2[k] + \left( \nabla_k f(w_t)\right)^2 + \sigma^2}}  + \beta \sqrt{\eta[k]} \sigma \Expt{ \frac{g^2_t[k]}{b_t^2[k]}}. \label{eq:bound_surrogate_eq3}
\end{eqnarray}
The last inequality follows from \ref{eq:BV}. Similarly, we again use Lemma \ref{lemma:AM-GM} with 
\begin{eqnarray*}
    \lambda & = & \frac{2}{\sqrt{b^2_{t-1}[k] + \left( \nabla_k f(w_t) \right)^2 + \sigma^2}}, \\
    a & = & \frac{\sigma \left| g_t[k]\right|}{b_t[k]},\\
    b & = & \frac{\left| \nabla_k f(w_t)\right|}{\sqrt{b^2_t[k] + \left( \nabla_k f(w_t)\right)^2 + \sigma^2}}
\end{eqnarray*}
and $\sqrt{b^2_t[k] + \left( \nabla_k f(w_t)\right)^2 + \sigma^2} \geq \sigma$ to get 
\begin{eqnarray}
    \beta \sqrt{\eta[k]} \Expt{\frac{\sigma \left| \nabla_k f(w_t) \right| \left| g_t[k] \right|}{b_t[k] \sqrt{b^2_{t-1}[k] + (\nabla_k f(w_t))^2 + \sigma^2}}} &\leq& \beta \sqrt{\eta[k]} \sigma \Expt{ \frac{g^2_t[k]}{b_t^2[k]}}\notag\\
    &&+ \frac{ \beta \sqrt{\eta[k]} \left( \nabla_k f(w_t) \right)^2}{4 \sqrt{b^2_{t-1}[k] + \left( \nabla f(w_t)\right)^2 + \sigma^2}}. \label{eq:bound_surrogate_eq4}
\end{eqnarray}
Therefore using \eqref{eq:bound_surrogate_eq3} and \eqref{eq:bound_surrogate_eq4} in \eqref{eq:bound_surrogate_eq3} we get 
\begin{eqnarray*}
    \Expt{\left( \frac{\beta \sqrt{\eta[k]}}{\sqrt{b_{t-1}^2[k] + \left(\nabla_k f(w_t) \right)^2 + \sigma^2}} - \nu_t[k] \right) \nabla_k f(w_t) g_t[k]}
    &\leq & 2\beta \sqrt{\eta[k]} \sigma \Expt{ \frac{g^2_t[k]}{b_t^2[k]}}\\
    && + \frac{ \beta \sqrt{\eta[k]} \left( \nabla_k f(w_t) \right)^2}{2 \sqrt{b^2_{t-1}[k] + \left( \nabla f(w_t)\right)^2 + \sigma^2}}.
\end{eqnarray*}
This completes the proof of this Lemma. 
\end{proof}

\begin{lemma}\label{lemma:sum_bound}
    \begin{eqnarray}\label{eq:sum_bound}
        \sum_{t = 0}^T \frac{g_t^2[k]}{b_t^2[k]} \leq \log \left( \frac{b_T^2[k]}{b_0^2[k]}\right) + 1
    \end{eqnarray}
\end{lemma}
\begin{proof} Using $b_t^2[k] = \sum_{\tau = 0}^t g_{\tau}^2[k]$ we have 
    \begin{eqnarray*}
        \sum_{t = 0}^T \frac{g_t^2[k]}{b_t^2[k]} &=& 1 + \sum_{t = 1}^T \frac{g_t^2[k]}{b_t^2[k]} \\
        & = & 1 + \sum_{t = 1}^T \frac{b_t^2[k] - b_{t -1}^2[k]}{b_t^2[k]} \\
        & = & 1 + \sum_{t = 1}^T \frac{1}{b_t^2[k]} \int_{b^2_{t-1}[k]}^{b_t^2[k]} dz \\
        & \leq & 1 + \sum_{t = 1}^T \int_{b^2_{t-1}[k]}^{b_t^2[k]} \frac{dz}{z} \\
        & = & 1 +  \int_{b^2_{0}[k]}^{b_T^2[k]} \frac{dz}{z} \\
        & = & 1 + \log \left( \frac{b_T^2[k]}{b_0^2[k]}\right).
    \end{eqnarray*}
    The inequality follows from the fact $\frac{1}{b_t^2[k]} \leq \frac{1}{z}$ when $b_{t-1}^2[k] \leq z \leq b_t^2[k]$. This completes the proof of the Lemma.
\end{proof}

\newpage
\section{Proof of Main Results}\label{section:main_results}

\subsection{Proof of Proposition~\ref{prop:invariance}}\label{proof:prop_invariance}

\begin{proposition}[Scale invariance]\label{prop:invariance_appendix}
    Suppose we solve problems \eqref{eq:GLM_opt_unscaled} and \eqref{eq:GLM_opt_scaled} using algorithm \eqref{eq:invariant_adagrad}. Then, the iterates $\hw_t$ and $\hw^V_t$ corresponding to \eqref{eq:GLM_opt_unscaled} and \eqref{eq:GLM_opt_scaled} follow: $\forall k \in [d]$
    \begin{eqnarray}
    \textstyle
        \hw_{t+1}[k] &=& \hw_t[k] - \tfrac{\beta m_t[k]}{\textstyle{\sum}_{\tau = 0}^t g_{\tau}^2[k]} g_t[k], \label{eq:prop_invariance_update1_appendix}\\
        \quad \hw^V_{t+1}[k] &= & \hw^V_t[k] - \tfrac{\beta m_t[k]}{\textstyle{\sum}_{\tau = 0}^t \left(g_{\tau}^V[k] \right)^2} g^V_t[k] \label{eq:prop_invariance_update2_appendix}
    \end{eqnarray}
    with $g_{\tau} = \varphi'_{i_{\tau}}(x_{i_{\tau}}^{\top}\hw_{\tau}) x_{i_{\tau}}$ and $g^V_{\tau} = \varphi'_{i_{\tau}}(x_{i_{\tau}}^{\top}V \hw_{\tau}) Vx_{i_{\tau}}$ for $i_{\tau}$ chosen uniformly from $[n]$, $\tau = 0,1,\ldots, t$, $t\geq 0$. Moreover, updates \eqref{eq:prop_invariance_update1_appendix} and \eqref{eq:prop_invariance_update2_appendix} satisfy
    \begin{eqnarray*}
    \textstyle
        \hw_t = V\hw^V_t, \quad Vg_t = g^V_t, \quad f \left(\hw_t \right) = f^V \left(\hw^V_t \right)
    \end{eqnarray*}
    for all $t \geq 0$ when $\hw_0 = \hw^V_0 = 0 \in \R^d$. Furthermore we have 
\begin{eqnarray}\label{eq:invariance_grad_appendix}
\textstyle
         \left\| g_t^V \right\|_{V^{-2}}^2 &=& \left\| g_t\right\|^2.    
    \end{eqnarray}
\end{proposition}
\begin{proof}
    First, we will show $\hw_t = V \hw^V_t$ and $V g_t = g^V_t$ using induction. Note that for $\tau = 1$ and $k \in [d]$, we get
     \begin{eqnarray*}
     \textstyle
         \hw_1[k] &=& \tfrac{-\beta m_0[k] \varphi'_{i_{0}}(0) x_{i_{0}}[k]}{\left( \varphi'_{i_{0}}(0) x_{i_{0}}[k] \right)^2} = \tfrac{-\beta m_0[k]}{\varphi'_{i_{0}}(0) x_{i_{0}}[k]}, \\
         \hw^V_1[k] &=& \tfrac{-\beta m_0[k]\varphi'_{i_{0}}(0) V_{kk}x_{i_{0}}[k]}{\left( \varphi'_{i_{0}}(0) V_{kk}x_{i_{0}}[k] \right)^2} = \tfrac{-\beta m_0[k]}{\varphi'_{i_{0}}(0) V_{kk}x_{i_{0}}[k]}.
     \end{eqnarray*}
     as $\hw_0 = \hw^V_0 = 0$. Therefore, we have $\forall k \in [d], \hw_1[k] = V_{kk}\hw^V_1[k]$. This can be equivalently written as $\hw_1 = V\hw^V_1$, as $V$ is a diagonal matrix. Then it is easy to check
\begin{equation}
\label{eq:prop_invariance1_appendix}
         Vg_1   =   \varphi'_{i_1} \left( x_{i_1}^{\top}\hw_1\right) Vx_{i_1} = \varphi'_{i_1} \left( x_{i_1}^{\top}V\hw^V_1\right) Vx_{i_1} = g_1^V,    
     \end{equation}
     where the second equality follows from $\hw_1 = V\hw^V_1$. Now, we assume the proposition holds for $\tau = 1, \cdots, t$. Then, we need to prove this proposition for $\tau = t+1$. Note that, from \eqref{eq:prop_invariance_update1} we have
     \begin{equation*}
     \textstyle
         \hw_{t+1}[k]   =   \hw_t[k] - \tfrac{\beta m_t[k]}{\sum_{\tau = 0}^t g^2_{\tau}[k]} g_t[k] = V_{kk} \hw^V_t[k] - \tfrac{\beta m_t[k] V_{kk}^2}{\sum_{\tau = 0}^t \left(g^V_{\tau}[k] \right)^2} \frac{g^V_t[k]}{V_{kk}} = V_{kk} \hw^V_{t+1}[k].
     \end{equation*}
     Here, the second last equality follows from $\hw_{\tau} = V\hw^V_{\tau}$ and $Vg_{\tau} = g^V_{\tau} \quad \forall \tau \in [t]$, while the last equality holds due to \eqref{eq:prop_invariance_update2_appendix}. Therefore, we have $\hw_{t+1} = V\hw^V_{t+1}$. Then similar to \eqref{eq:prop_invariance1_appendix} we get $Vg_{t+1} = g^V_{t+1}$ using $\hw_{t+1} = V\hw^V_{t+1}$. Again, using $\hw_{t} = V\hw^V_{t}$, we can rewrite $f(\hw_{t})$ as follow
     \begin{equation*}
     \textstyle
         f(\hw_{t}) = \frac{1}{n} \sum_{i = 1}^n \varphi_i \left( x_i^{\top} \hw_t\right) \\
         = \frac{1}{n} \sum_{i = 1}^n \varphi_i \left( x_i^{\top} V \hw^V_t\right) \\
         = f^V \left(\hw^V_t \right).
     \end{equation*}
     The last equality follows from \eqref{eq:GLM_opt_scaled}. This proves $f(\hw_{t}) = f^V \left(\hw^V_t \right)$. Finally using $Vg_t = g^V_t$ we get
     \begin{equation*}
     \textstyle
         \left\| g_t^V \right\|_{V^{-2}}^2 = \left(g_t^V \right)^{\top} V^{-2} g_t^V   
         = g_t^{\top} V V^{-2} V g_t  
         = \left\| g_t\right\|^2.    
    \end{equation*}
     This completes the proof of Proposition \ref{prop:invariance}.
\end{proof}
\newpage

\subsection{Proof of Lemma \ref{lemma:decreasing_step}}
\vspace{.5cm}
\begin{lemma}[Decreasing step size]
For $\nu_t[k]$ defined in \eqref{eq:algo_coordinate} we have
    \begin{equation*}
        \nu_{t+1}[k] \leq \nu_t[k] \qquad \forall k \in [d].
    \end{equation*}
\end{lemma}
\begin{proof}
    We want to show that $\nu_{t+1}[k] \leq \nu_t[k]$. Taking square and rearranging the terms \eqref{eq:decreasing_step} is equivalent to proving
    \begin{eqnarray}
        b_t^4[k] m_{t+1}^2[k] \leq b_{t+1}^4[k] m_t^2[k]. \label{eq:decreasing_step_eq1}
    \end{eqnarray}
    Using the expansion of $m^2_{t+1}[k], b^2_{t+1}[k]$, LHS of \eqref{eq:decreasing_step_eq1} can be expanded as follow
    \begin{eqnarray}\label{eq:decreasing_step_eq2}
        b_t^4[k] m_{t+1}^2[k] & = & b_t^4[k] \left( m_t^2[k] + \eta[k] g^2_{t+1}[k] + \frac{g_{t+1}^2[k]}{b_t^2[k] + g^2_{t+1}[k]} \right).
    \end{eqnarray}
    Similarly, the RHS of \eqref{eq:decreasing_step_eq1} can be expanded to
    \begin{eqnarray}
        b_{t+1}^4[k] m_t^2[k] & = & m_t^2[k] \left( b_t^2[k] + g_{t+1}^2[k] \right)^2 \notag \\
        & = & m_t^2[k] b_t^4[k] + m^2_t[k] g_{t+1}^4[k] + 2 m_t^2[k] g_{t+1}^2[k]b_t^2[k] \label{eq:decreasing_step_eq3}.
    \end{eqnarray}
    Therefore using \eqref{eq:decreasing_step_eq2} and \eqref{eq:decreasing_step_eq3}, inequality \eqref{eq:decreasing_step_eq1} is equivalent to 
    \begin{eqnarray}
        b_t^4[k] \left( m_t^2[k] + \eta[k] g^2_{t+1}[k] + \frac{g_{t+1}^2[k]}{b_t^2[k] + g^2_{t+1}[k]} \right) &\leq & m_t^2[k] b_t^4[k] + m^2_t[k] g_{t+1}^4[k]\notag\\
        &&+ 2 m_t^2[k] g_{t+1}^2[k]b_t^2[k]. \label{eq:decreasing_step_eq4}
    \end{eqnarray}
    Now subtracting $m_t^2[k] b_t^4[k]$ from both sides of \eqref{eq:decreasing_step_eq4} and then multiplying both sides by $b_t^2[k] + g^2_{t+1}[k]$, \eqref{eq:decreasing_step_eq4} is equivalent to
    \begin{eqnarray}
        \eta[k] g_{t+1}^2[k] b_t^6[k] + \eta[k] g_{t+1}^4[k] b_t^4[k] + g_{t+1}^2[k] b_t^4[k] &\leq & m_t^2[k] g_{t+1}^4[k] b_t^2[k] + 2 m_t^2[k] g_{t+1}^2[k]b^4_{t}[k] \notag\\
        && + m_t^2[k]g^6_{t+1}[k] + 2 m_t^2[k] g_{t+1}^4[k] b_t^2[k].  \label{eq:decreasing_step_eq5}
    \end{eqnarray}
    Therefore, proving \eqref{eq:decreasing_step} is equivalent to proving \eqref{eq:decreasing_step_eq5}. Note that, from the expansion $m_t^2[k] = \eta[k] b_t^2[k] + \sum_{\tau = 0}^t \frac{g_t^2[k]}{b_t^2[k]}$, we have $m_t^2[k] \geq \frac{g_0^2[k]}{b_0^2[k]} = 1$ and $m_t^2[k] \geq \eta[k] b_t^2[k]$. Then using $m_t^2[k] \geq 1$ we get 
    \begin{eqnarray}
         g_{t+1}^4[k] b_t^2[k] \leq m_t^2[k] g_{t+1}^4[k] b_t^2[k]. \label{eq:decreasing_step_eq6}
    \end{eqnarray}
    Again, using $m_t^2[k] \geq \eta[k] b_{t}^2[k]$, we have 
    \begin{eqnarray}
         \eta[k] g_{t+1}^2[k] b_t^6[k] + \eta[k] g_{t+1}^4[k] b_t^4[k] &\leq & m_t^2[k] g_{t+1}^2[k]b^4_{t}[k] + m_t^2[k] g_{t+1}^4[k] b_t^2[k]. \label{eq:decreasing_step_eq7}
    \end{eqnarray}
    Then adding \eqref{eq:decreasing_step_eq6} and \eqref{eq:decreasing_step_eq7} we get
    \begin{eqnarray}
        \eta[k] g_{t+1}^2[k] b_t^6[k] + \eta[k] g_{t+1}^4[k] b_t^4[k] + g_{t+1}^2[k] b_t^4[k] &\leq & m_t^2[k] g_{t+1}^4[k] b_t^2[k] + 2 m_t^2[k] g_{t+1}^2[k]b^4_{t}[k]. \label{eq:decreasing_step_eq8}
    \end{eqnarray}
    Therefore, \eqref{eq:decreasing_step_eq5} is true due to \eqref{eq:decreasing_step_eq8} and $ m_t^2[k]g^6_{t+1}[k] + 2 m_t^2[k] g_{t+1}^4[k] b_t^2[k] \geq 0$. This completes the proof of the Lemma. 
\end{proof}

\newpage
\subsection{Proof of Theorem \ref{theorem:deterministic_nonconvex}}
\vspace{.5cm}
\begin{theorem}
    Suppose $f$ is $L$-smooth, $g_t = \nabla f(w_t)$ and $\eta, \beta$ are chosen such that $\nu_0[k] \leq \frac{1}{L}$ for all $k \in [d]$. Then for \eqref{eq:algo_coordinate} we have
    \begin{eqnarray*}
    \min_{t\leq T} \left\| \nabla f(w_t) \right\|^2 \leq \frac{1}{T+1} \left(  \sum_{k = 1}^d b_0[k] + \frac{2(f(w_0) - f_*)}{\sqrt{\eta} \beta}\right)^2. 
\end{eqnarray*}
\end{theorem}
\begin{proof}
    Suppose $g_t = \nabla f(w_t)$. Then using the smoothness of $f$ we get
\begin{eqnarray*}
    f(w_{T+1}) &\leq & f(w_T) + \left \langle g_T, w_{T+1} - w_T\right \rangle + \frac{L}{2} \left \|w_{T+1} - w_T\right \|^2 \\
    & = & f(w_T) + \sum_{k=1}^d g_T[k] \left( w_{T+1}[k] - w_T[k] \right) + \frac{L}{2} \sum_{k=1}^d \left( w_{T+1}[k] - w_T[k] \right)^2 \\
    & = & f(w_T) - \sum_{k = 1}^d \nu_T[k] g_T^2[k] + \frac{L}{2} \sum_{k = 1}^d \nu_T^2[k] g_T^2[k] \\
    & = & f(w_T) - \sum_{k = 1}^d \nu_T[k] \left( 1 - \nu_T[k]\frac{L}{2}\right) g_T^2[k].
\end{eqnarray*}
Then using this bound recursively we get
\begin{eqnarray*}
    f(w_{T+1}) &\leq & f(w_0) - \sum_{t = 0}^T\sum_{k = 1}^d \nu_t[k] \left( 1 - \nu_t[k]\frac{L}{2}\right) g_t^2[k]. \label{eq:deterministic_nonconvex_eq1}
\end{eqnarray*}
Note that, we initialized \algname{KATE} such that $\nu_0[k] \leq \frac{1}{L} \forall k \in [d]$. Therefore using Lemma \ref{lemma:decreasing_step} we have $\nu_t[k] \leq \frac{1}{L}$, which is equivalent to $1 - \nu_t[k] \frac{L}{2} \geq \frac{1}{2}$ for all $k \in [d]$. Hence from \eqref{eq:deterministic_nonconvex_eq1} we have
\begin{eqnarray*}
    f(w_{T+1}) &\leq & f(w_0) - \sum_{t = 0}^T \sum_{k = 1}^d \frac{\nu_t[k]}{2} g_t^2[k].
\end{eqnarray*}
Then rearranging the terms and using $f(w_{T+1}) \geq f_*$ we get 
\begin{eqnarray}\label{eq:deterministic_nonconvex_eq2}
    \sum_{t = 0}^T \sum_{k = 1}^d \frac{\nu_t[k]}{2} g_t^2[k] \leq f(w_0) - f_*.
\end{eqnarray}
Then from \eqref{eq:deterministic_nonconvex_eq2} and $m_t[k] \geq \sqrt{\eta_0}  b_t[k]$ we get
\begin{eqnarray}
    \sum_{t = 0}^T \sum_{k = 1}^d \frac{g_t^2[k]}{b_t[k]} \leq \frac{2(f(w_0) - f_*)}{\sqrt{\eta_0} \beta}. \label{eq:deterministic_nonconvex_eq3}
\end{eqnarray}
Now from the definition of $b_t^2[k]$, we have $b^2_t[k] = b_{t-1}^2[k] + g_t^2[k]$. This can be rearranged to get
\begin{eqnarray}
    b_T[k] &=& b_{T-1}[k] + \frac{g_T^2[k]}{b_T[k] + b_{T-1}[k]} \notag\\
    &\leq & b_{T-1}[k] + \frac{g_T^2[k]}{b_T[k]}  \label{eq:deterministic_nonconvex_eq4}\\
    &\leq & b_{0}[k] + \sum_{t = 0}^T \frac{g_t^2[k]}{b_t[k]}.  \label{eq:deterministic_nonconvex_eq5} 
\end{eqnarray}
Here the last inequality  \eqref{eq:deterministic_nonconvex_eq5} follows from recursive use of \eqref{eq:deterministic_nonconvex_eq4}. Then, taking squares on both sides and summing over $k \in [d]$ we get
\begin{eqnarray}
    \sum_{k=1}^d b_T^2[k] &\leq & \sum_{k = 1}^d \left( b_0[k] + \sum_{t = 0}^T \frac{g_t^2[k]}{b_t[k]} \right)^2 \notag\\
    & \leq & \left(  \sum_{k = 1}^d b_0[k] + \sum_{t = 0}^T  \sum_{k = 1}^d \frac{g_t^2[k]}{b_t[k]} \right)^2 \notag\\
    & \leq & \left(  \sum_{k = 1}^d b_0[k] + \frac{2(f(w_0) - f_*)}{\sqrt{\eta_0} \beta}\right)^2. \label{eq:deterministic_nonconvex_eq6}
\end{eqnarray}
The second inequality follows from $b_0[k] + \sum_{t = 0}^T \frac{g_t^2[k]}{b_t[k]} \geq 0$ for all $k \in [d]$ and the last inequality from \eqref{eq:deterministic_nonconvex_eq3}. Now note that $\sum_{t = 0}^T \left\| g_t\right\|^2 = \sum_{t = 0}^T \sum_{k = 1}^d g_t^2[k] = \sum_{k = 1}^d b^2_t[k]$. Therefore dividing both sides of \eqref{eq:deterministic_nonconvex_eq6} by $T+1$, we get
\begin{eqnarray*}
    \min_{t \leq T} \left\| \nabla f(w_t)\right\|^2 \leq \frac{1}{T+1} \left(  \sum_{k = 1}^d b_0[k] + \frac{2(f(w_0) - f_*)}{\sqrt{\eta_0} \beta}\right)^2. 
\end{eqnarray*}
This completes the proof of the theorem.
\end{proof}

\newpage
\subsection{Proof of Theorem \ref{theorem:stoch_nonconvex}}
\vspace{.5cm}
\begin{theorem}
    Suppose $f$ is a $L$-smooth function and $g_t$ is an unbiased estimator of $\nabla f(w_t)$ such that \ref{eq:BV} holds. Moreover, we assume $\| \nabla f(w_t)\|^2 \leq \gamma^2$ for all $t$. Then \textcolor{red}{\algname{KATE}} satisfies
    \begin{eqnarray*}
        \min_{t \leq T} \E{\left\| \nabla f(w_t) \right\|} &\leq& \left( \frac{\|g_0\|}{T} + \frac{2(\gamma + \sigma)}{\sqrt{T}}\right)^{\nicefrac{1}{2}} \sqrt{\frac{2 \mathcal{C}_f}{\beta \sqrt{\eta_0}}}
    \end{eqnarray*}
    where
    \begin{eqnarray*}
        \mathcal{C}_f = f(w_0) - f_* + \sum_{k = 1}^d \left(2\beta \sqrt{\eta[k]} \sigma + \frac{\beta^2 \eta[k] L}{2} + \frac{\beta^2 L}{2 g_0^2[k]} \right) \left( \log \left( \frac{(\sigma^2 + \gamma^2) T}{g_0^2[k]}\right) + 1  \right).
    \end{eqnarray*}
\end{theorem}

\begin{proof}
    Using smoothness, we have
    \begin{eqnarray*}
    f(w_{t+1}) &\leq & f(w_t) + \left \langle \nabla f(w_t), w_{t+1} - w_t\right \rangle + \frac{L}{2} \left \|w_{t+1} - w_t \right \|^2 \\
    & = & f(w_t) + \sum_{k=1}^d \nabla_k f(w_t) \left( w_{t+1}[k] - w_t[k] \right) + \frac{L}{2} \sum_{k=1}^d \left( w_{t+1}[k] - w_t[k] \right)^2 \\
    & = & f(w_t) - \sum_{k = 1}^d \nu_t[k] \nabla_k f(w_t)g_t[k] + \frac{L}{2} \sum_{k = 1}^d \nu_t^2[k] g_t^2[k].
\end{eqnarray*}
Then, taking the expectation conditioned on $w_t$, we have
\begin{eqnarray*}
    \Expt{f(w_{t+1})} & \leq & f(w_t) - \sum_{k = 1}^d \Expt{ \nu_t[k] \nabla_k f(w_t) g_t[k]} + \frac{L}{2} \sum_{k = 1}^d \Expt{\nu_t^2[k] g_t^2[k]} \\
    & = & f(w_t) - \sum_{k = 1}^d \Expt{\nu_t[k] \nabla_k f(w_t) g_t[k]} + \frac{L}{2} \sum_{k = 1}^d \Expt{\nu_t^2[k] g_t^2[k]} \\
    && - \sum_{k = 1}^d \frac{\beta \sqrt{\eta[k]}}{\sqrt{b^2_{t-1}[k] + \left(\nabla_k f(w_t) \right)^2 + \sigma^2}} \Expt{\nabla_k f(w_t) \left( \nabla_k f(w_t) - g_t[k] \right)} \\
    & = & f(w_t) + \sum_{k = 1}^d \Expt{\left( \frac{\beta \sqrt{\eta[k]}}{\sqrt{b_{t-1}^2[k] + \left(\nabla_k f(w_t) \right)^2 + \sigma^2}} - \nu_t[k] \right) \nabla_k f(w_t) g_t[k]} \\
    && + \frac{L}{2} \sum_{k = 1}^d \Expt{\nu_t^2[k] g_t^2[k] } - \sum_{k = 1}^d \frac{\beta \sqrt{\eta[k]} \left(\nabla_k f(w_t) \right)^2}{\sqrt{b^2_{t-1}[k] + \left(\nabla_k f(w_t) \right)^2 + \sigma^2}}. 
\end{eqnarray*}
The second last equality follows from $\Expt{\nabla_k f(w_t) \left( \nabla_k f(w_t) - g_t[k] \right)} = \nabla_k f(w_t) \left( \nabla_k f(w_t) - \Expt{g_t[k]} \right) = 0$. Now we use \eqref{eq:bound_surrogate} to get

\begin{eqnarray*}
    \Expt{f(w_{t+1})} & \leq & f(w_t) + \sum_{k = 1}^d 2\beta \sqrt{\eta[k]} \sigma \Expt{\frac{g^2_t[k]}{b_t^2[k]}} + \frac{L}{2} \sum_{k = 1}^d \Expt{\nu_t^2[k] g_t^2[k]}\\
    &&- \sum_{k = 1}^d \frac{\beta \sqrt{\eta[k]} \left(\nabla_k f(w_t) \right)^2}{2 \sqrt{b^2_{t-1}[k] + \left(\nabla_k f(w_t) \right)^2 + \sigma^2}}. 
\end{eqnarray*}
Then rearranging the terms we have
\begin{eqnarray*}
    \sum_{k = 1}^d \frac{\beta \sqrt{\eta[k]} \left(\nabla_k f(w_t) \right)^2}{2 \sqrt{b^2_{t-1}[k] + \left(\nabla_k f(w_t) \right)^2 + \sigma^2}} & \leq & f(w_t) - \Expt{f(w_{t+1})} + \sum_{k = 1}^d 2\beta \sqrt{\eta[k]} \sigma \Expt{ \frac{g^2_t[k]}{b_t^2[k]}}\\
    &&+ \frac{L}{2} \sum_{k = 1}^d \Expt{ \nu_t^2[k] g_t^2[k]}.
\end{eqnarray*}
Now we take the total expectations to derive
\begin{eqnarray*}
    \sum_{k = 1}^d \Exp{\frac{\beta \sqrt{\eta[k]} \left(\nabla_k f(w_t) \right)^2}{2 \sqrt{b^2_{t-1}[k] + \left(\nabla_k f(w_t) \right)^2 + \sigma^2}} } & \leq & \Exp{f(w_t)} - \Exp{f(w_{t+1})} + \sum_{k = 1}^d 2\beta \sqrt{\eta[k]} \sigma \Exp{\frac{g^2_t[k]}{b_t^2[k]}}\\
    &&+ \frac{L}{2} \sum_{k = 1}^d \Exp{\nu_t^2[k] g_t^2[k]}.
\end{eqnarray*}
The above inequality holds for any $t$. Therefore summing up from $t = 0$ to $t = T$ and using $f(w_{T+1}) \geq f_*$ we get
\begin{eqnarray}\label{eq:stoch_nonconvex1}
    \sum_{t= 0}^T \sum_{k = 1}^d \Exp{\frac{\beta \sqrt{\eta[k]} \left(\nabla_k f(w_t) \right)^2}{2 \sqrt{b^2_{t-1}[k] + \left(\nabla_k f(w_t) \right)^2 + \sigma^2}} } & \leq & f(w_0) - f_* + \sum_{t = 0}^T \sum_{k = 1}^d 2\beta \sqrt{\eta[k]} \sigma \Exp{ \frac{g^2_t[k]}{b_t^2[k]}} \notag \\
    && + \frac{L}{2} \sum_{t = 0}^T \sum_{k = 1}^d \Exp{ \nu_t^2[k] g_t^2[k]}.
\end{eqnarray}
Note that, using the expansion of $\nu_t^2[k]$ we have
\begin{eqnarray}
\nu_t^2[k] & = &\frac{\beta^2\eta[k] b_t^2[k] + \beta^2 \sum_{j = 0}^t \frac{g_j^2[k]}{b_j^2[k]}}{b_t^4[k]} \notag\\
&= & \frac{\beta^2 \eta[k]}{b_t^2[k]} + \frac{\beta^2}{b_t^4[k]} \sum_{j = 0}^t \frac{g_j^2[k]}{b_j^2[k]} \notag\\
&\leq & \frac{\beta^2 \eta[k]}{b_t^2[k]} + \frac{\beta^2}{b_t^4[k] b_0^2[k]} \sum_{j = 0}^t g_j^2[k] \label{eq:stoch_nonconvex2}\\
&= & \frac{\beta^2 \eta[k]}{b_t^2[k]} + \frac{\beta^2}{b_t^2[k] g_0^2[k]}. \label{eq:stoch_nonconvex3}   
\end{eqnarray}
Here \eqref{eq:stoch_nonconvex2} follows from $b_j^2[k] \geq b_0^2[k]$ and \eqref{eq:stoch_nonconvex3} from $b_t^2[k] = \sum_{j = 0}^t g_j^2[k]$.
Then using \eqref{eq:stoch_nonconvex3} in \eqref{eq:stoch_nonconvex1} we derive

\begin{align*}
    \sum_{t= 0}^T \sum_{k = 1}^d \Exp{\frac{\beta \sqrt{\eta[k]} \left(\nabla_k f(w_t) \right)^2}{2 \sqrt{b^2_{t-1}[k] + \left(\nabla_k f(w_t) \right)^2 + \sigma^2}}} & \leq  f(w_0) - f_* + \sum_{t = 0}^T \sum_{k = 1}^d \left(2\beta \sqrt{\eta[k]} \sigma + \frac{\beta^2 \eta[k] L}{2} + \frac{\beta^2 L}{2 g_0^2[k]} \right) \Exp{ \frac{g^2_t[k]}{b_t^2[k]}} \\
    &\leq  f(w_0) - f_* \\
    & + \sum_{k = 1}^d \left(2\beta \sqrt{\eta[k]} \sigma + \frac{\beta^2 \eta[k] L}{2} + \frac{\beta^2 L}{2 g_0^2[k]} \right) \Exp{\log \left( \frac{b^2_T[k]}{b_0^2[k]}\right) + 1}.
\end{align*}
Here the last inequality follows from \eqref{eq:sum_bound}. Now using Jensen's Inequality \eqref{eq:Jensen} with $\Psi(z) = \log(z)$ we have
\begin{align*}
    \sum_{t= 0}^T \sum_{k = 1}^d \Exp{\frac{\beta \sqrt{\eta[k]} \left(\nabla_k f(w_t) \right)^2}{2 \sqrt{b^2_{t-1}[k] + \left(\nabla_k f(w_t) \right)^2 + \sigma^2}}} &\leq  f(w_0) - f_* \\
    & + \sum_{k = 1}^d \left(2\beta \sqrt{\eta[k]} \sigma + \frac{\beta^2 \eta[k] L}{2} + \frac{\beta^2 L}{2 g_0^2[k]} \right) \left( \log \left( \frac{\Exp{b^2_T[k]}}{b_0^2[k]}\right) + 1  \right).
\end{align*}
Now note that $\Exp{b_T^2[k]} = \sum_{t = 0}^T \Exp{g_t^2[k]} = \sum_{t = 0}^T \Exp{g_t[k] - \nabla_k f(w_t)}^2 + \left( \nabla_k f(w_t)\right)^2 \leq (\sigma^2 + \gamma^2) T$. Therefore, we have the bound 
\begin{align}
    \sum_{t= 0}^T \sum_{k = 1}^d \Exp{\frac{\beta \sqrt{\eta[k]} \left(\nabla_k f(w_t) \right)^2}{2 \sqrt{b^2_{t-1}[k] + \left(\nabla_k f(w_t) \right)^2 + \sigma^2}}} &\leq & f(w_0) - f_* + 2\beta \sigma \sum_{k = 1}^d \sqrt{\eta[k]} \log \left( \frac{e(\sigma^2 + \gamma^2) T}{b_0^2[k]}\right) \notag \\
     && + \sum_{k = 1}^d \left(\frac{\beta^2 \eta[k] L}{2} + \frac{\beta^2 L}{2 g_0^2[k]} \right) \log \left( \frac{e(\sigma^2 + \gamma^2) T}{b_0^2[k]}\right).  \label{eq:stoch_nonconvex4}
\end{align}
Here the RHS is exactly $\mathcal{C}_f$. Using \eqref{eq:inequality} we have
\begin{eqnarray}
    \sum_{k=1}^d \frac{\left(\nabla_k f(w_t) \right)^2}{\sqrt{b^2_{t-1}[k] + \left(\nabla_k f(w_t) \right)^2 + \sigma^2}} & \geq & \frac{\sum_{k = 1}^d \left( \nabla_k f(w_t)\right)^2}{\sqrt{\sum_{k =1}^d b^2_{t-1}[k] + \left( \nabla_k f(w_t)\right)^2 + \sigma^2}} \notag\\
    & = & \frac{\left\| \nabla f(w_t) \right\|^2}{\sqrt{ \left\| b_{t-1}\right\|^2 +  \left\| \nabla f(w_t)\right\|^2 + d \sigma^2}}. \label{eq:stoch_nonconvex5} 
\end{eqnarray}
Therefore using \eqref{eq:stoch_nonconvex5} in \eqref{eq:stoch_nonconvex4} we arrive at
\begin{eqnarray}
    \sum_{t = 0}^T \Exp{\frac{\left\| \nabla f(w_t) \right\|^2}{\sqrt{ \left\| b_{t-1}\right\|^2 +  \left\| \nabla f(w_t)\right\|^2 + d \sigma^2}}} & \leq & \frac{2 \mathcal{C}_f}{\beta \sqrt{\eta_0}}. \label{eq:stoch_nonconvex6}
\end{eqnarray}
Now we use Holder's Inequality~\eqref{eq:holder} $\frac{\E(XY)}{\left(\E{\left| Y \right|^3} \right)^{\frac{1}{3}}} \leq \left( \E|X|^{\frac{3}{2}} \right)^{\frac{2}{3}}$ with
\begin{eqnarray*}
    X = \left( \frac{\left\| \nabla f(w_t) \right\|^2}{\sqrt{ \left\| b_{t-1}\right\|^2 +  \left\| \nabla f(w_t)\right\|^2 + d \sigma^2}} \right)^{\frac{2}{3}} \quad \text{ and } \quad Y = \left( \sqrt{ \left\| b_{t-1}\right\|^2 +  \left\| \nabla f(w_t)\right\|^2 + d \sigma^2} \right)^{\frac{2}{3}}
\end{eqnarray*}
to get a lower bound on LHS of \eqref{eq:stoch_nonconvex6}: 
\begin{eqnarray}
    \Exp{\frac{\left\| \nabla f(w_t) \right\|^2}{\sqrt{ \left\| b_{t-1}\right\|^2 +  \left\| \nabla f(w_t)\right\|^2 + d \sigma^2}}} & \geq & \frac{\Exp{\| \nabla f(w_t)\|^{\frac{4}{3}}}^{\frac{3}{2}}}{\sqrt{ \E \left(\left\| b_{t-1}\right\|^2 +  \left\| \nabla f(w_t)\right\|^2 + d \sigma^2 \right)}} \notag \\
    & \geq & \frac{\Exp{\| \nabla f(w_t)\|^{\frac{4}{3}}}^{\frac{3}{2}}}{\sqrt{\left\| b_{0}\right\|^2 + 2t(\gamma^2 + d\sigma^2) }}. \label{eq:stoch_nonconvex7}
\end{eqnarray}
Therefore from \eqref{eq:stoch_nonconvex6} and \eqref{eq:stoch_nonconvex7} we get
\begin{eqnarray*}
    \frac{T}{\sqrt{\|b_0\|^2 + 2T(\gamma^2 + d\sigma^2)}} \min_{t \leq T} \Exp{\| \nabla f(w_t)\|^{\frac{4}{3}}}^{\frac{3}{2}} & \leq & \frac{2 \mathcal{C}_f}{\beta \sqrt{\eta_0}}.
\end{eqnarray*}
Then multiplying both sides by $\frac{\|b_0\| + \sqrt{2T} (\gamma + \sqrt{d} \sigma)}{T}$ we have
\begin{eqnarray*}
    \min_{t \leq T} \Exp{\| \nabla f(w_t)\|^{\frac{4}{3}}}^{\frac{3}{2}} & \leq & \left( \frac{\|b_0\|}{T} + \frac{2(\gamma + \sigma)}{\sqrt{T}}\right) \frac{2 \mathcal{C}_f}{\beta \sqrt{\eta_0}}.
\end{eqnarray*}
Here we use  $\Exp{\left\| \nabla f(w_t) \right\|}^{\frac{4}{3}} \leq \Exp{\left\| \nabla f(w_t)\right\|^{\frac{4}{3}}}$ (follows from Jensen's Inequality \eqref{eq:Jensen} with $\Psi(z) = z^{\nicefrac{4}{3}}$) in the above equation to get
\begin{eqnarray*}
    \min_{t \leq T} \Exp{\left\| \nabla f(w_t)\right\|}^2 & \leq & \left( \frac{\|b_0\|}{T} + \frac{2(\gamma + \sigma)}{\sqrt{T}}\right) \frac{2 \mathcal{C}_f}{\beta \sqrt{\eta_0}}.
\end{eqnarray*}
This completes the proof of the Theorem.
\end{proof}

\newpage

\section{Additional Experiments: Scale-Invariance Verification}\label{sec:scale_invariance}

In this experiment, we implement $\algname{\textcolor{PineGreen}{KATE}}$ on problems \eqref{eq:GLM_opt_unscaled} (for unscaled data) and \eqref{eq:GLM_opt_scaled} (for scaled data) with $$\varphi_i(z) = \log \left(1 + \exp{\left( -y_i z\right)} \right).$$ 
We generate the data similar to Section \ref{sec:robustness}. We run $\algname{\textcolor{PineGreen}{KATE}}$ for $10,000$ iterations with mini-batch size $10$, $\eta = \nicefrac{1}{\left( \nabla f(w_0)\right)^2}$ and plot functional value $f(w_t)$ and accuracy in Figures \ref{fig:scale_fval} and \ref{fig:scale_accuracy}. We use the proportion of correctly classified samples to compute accuracy, i.e. $\frac{1}{n} \sum_{i = 1}^n \mathbf{1}_{\left\{y_i x_i^{\top}w_t \geq 0\right\}}$. 

In plots \ref{fig:scale_fval} and \ref{fig:scale_accuracy}, the functional value and accuracy of $\algname{\textcolor{PineGreen}{KATE}}$ coincide, which aligns with our theoretical findings (Proposition \ref{prop:invariance}). Figure \ref{fig:scale_grad} plots $\left\| \nabla f(w_t) \right\|^2$ and $\left\| \nabla f(w_t) \right\|_{V^{-2}}^2$ for unscaled and scaled data respectively. Here, \eqref{eq:invariance_grad} explains the identical values taken by the corresponding gradient norms of \algname{KATE} iterates for the scaled and unscaled data. Similarly, in Figure \ref{fig:scale_plots_adagrad}, we compare the performance of \algname{AdaGrad} on scaled and un-scaled data. This figure illustrates the lack of the scale-invariance for \algname{AdaGrad}.

\begin{figure}[H]
\centering
\begin{subfigure}[b]{.32\textwidth}
    \centering
    \includegraphics[width=\textwidth]{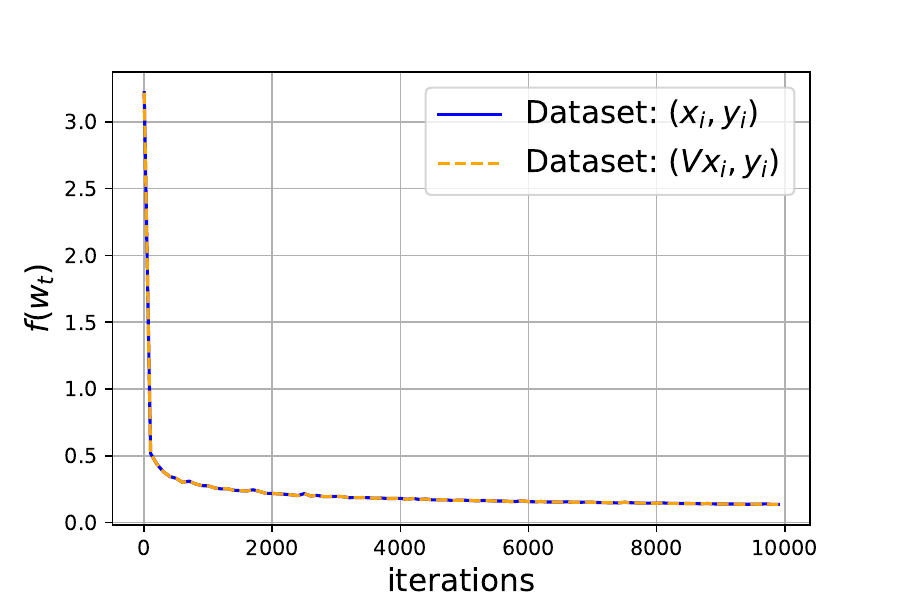}
    \caption{ Plot of $f(w_t)$}\label{fig:scale_fval}
\end{subfigure}
\begin{subfigure}[b]{0.32\textwidth}
    \centering
    \includegraphics[width=\textwidth]{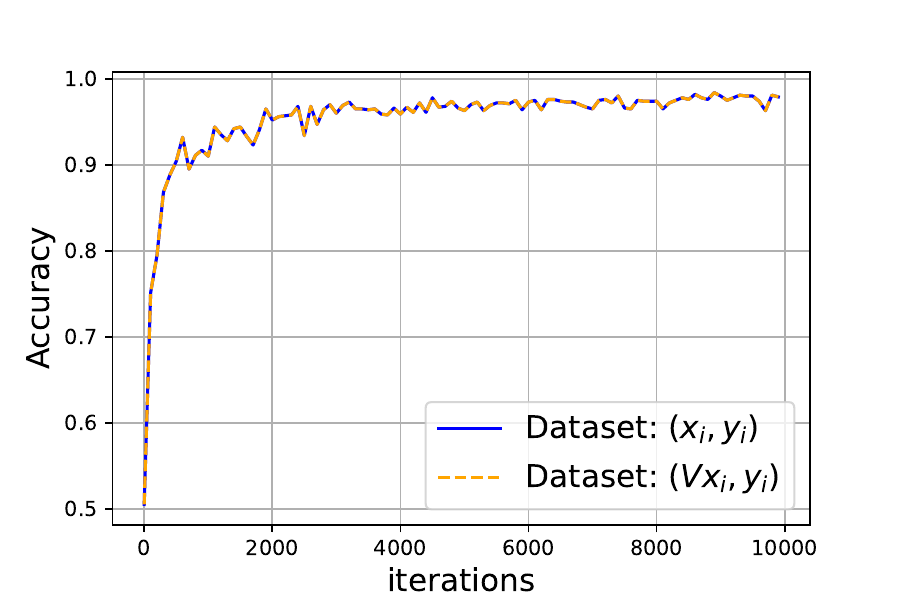}
    \caption{ Plot of Accuracy}\label{fig:scale_accuracy}
\end{subfigure}
\begin{subfigure}[b]{0.32\textwidth}
    \centering
    \includegraphics[width=\textwidth]{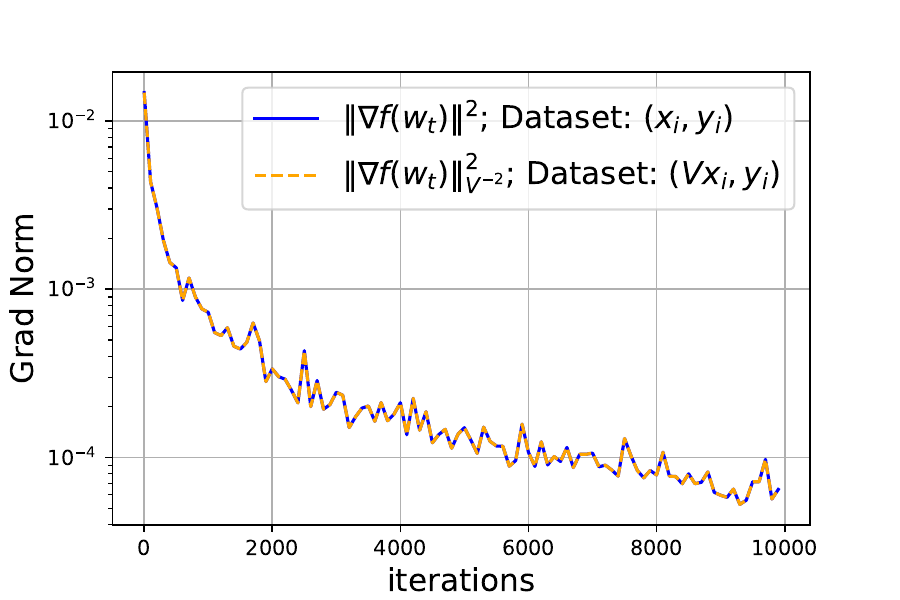}
    \caption{Plot of Gradient Norm}\label{fig:scale_grad}
\end{subfigure}
    \caption{Comparison of \algname{\textcolor{PineGreen}{KATE}} on scaled and un-scaled data. Figures \ref{fig:scale_fval}, and \ref{fig:scale_accuracy} plot the functional value $f(w_t)$ and accuracy on scaled and unscaled data, respectively. Figure \ref{fig:scale_grad} plots $\left\| \nabla f(w_t) \right\|^2$ and $\left\| \nabla f(w_t) \right\|_{V^{-2}}^2$ for unscaled and scaled data respectively.} \label{fig:scale_plots}
\end{figure}

\begin{figure}[H]
\centering
\begin{subfigure}[b]{.32\textwidth}
    \centering
    \includegraphics[width=\textwidth]{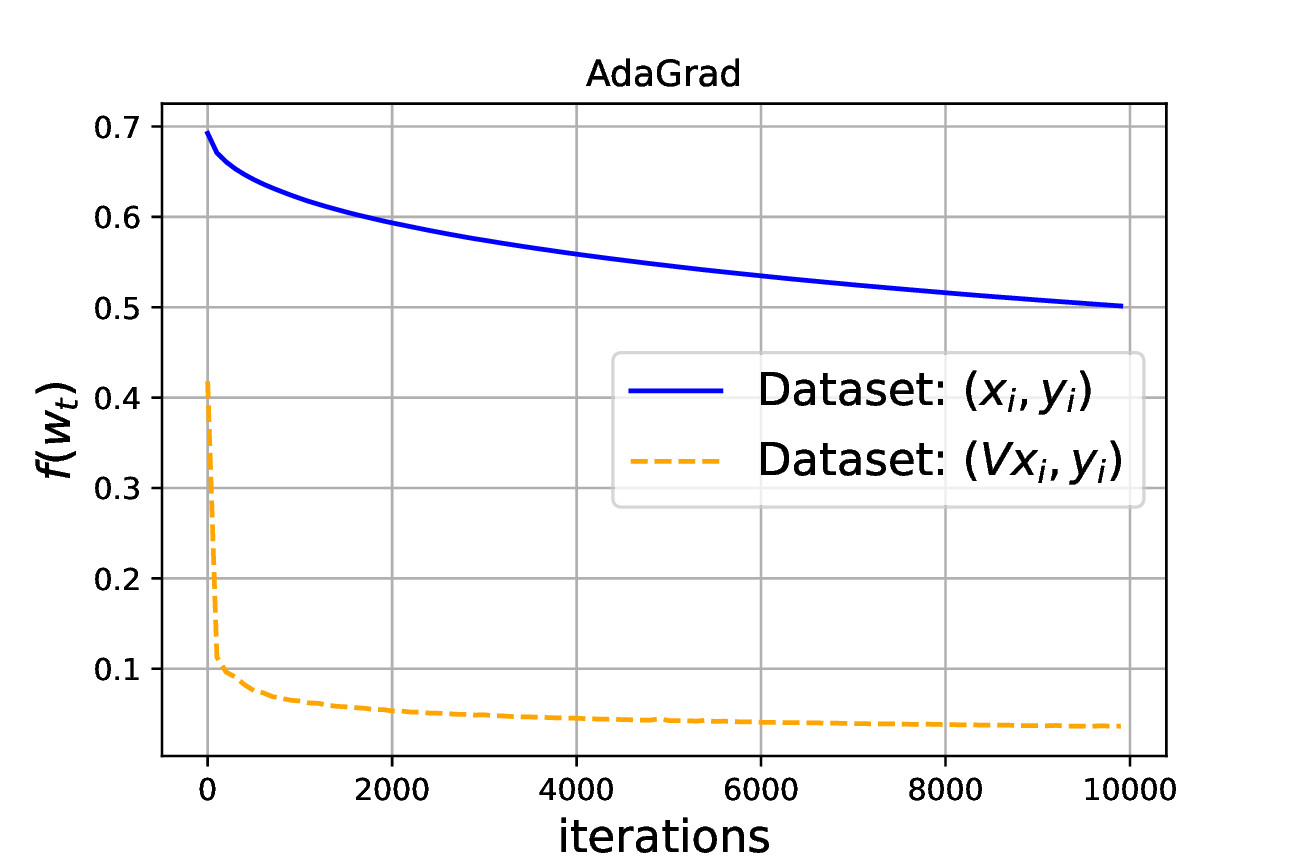}
    \caption{ Plot of $f(w_t)$}\label{fig:scale_fval_adagrad}
\end{subfigure}
\begin{subfigure}[b]{0.32\textwidth}
    \centering
    \includegraphics[width=\textwidth]{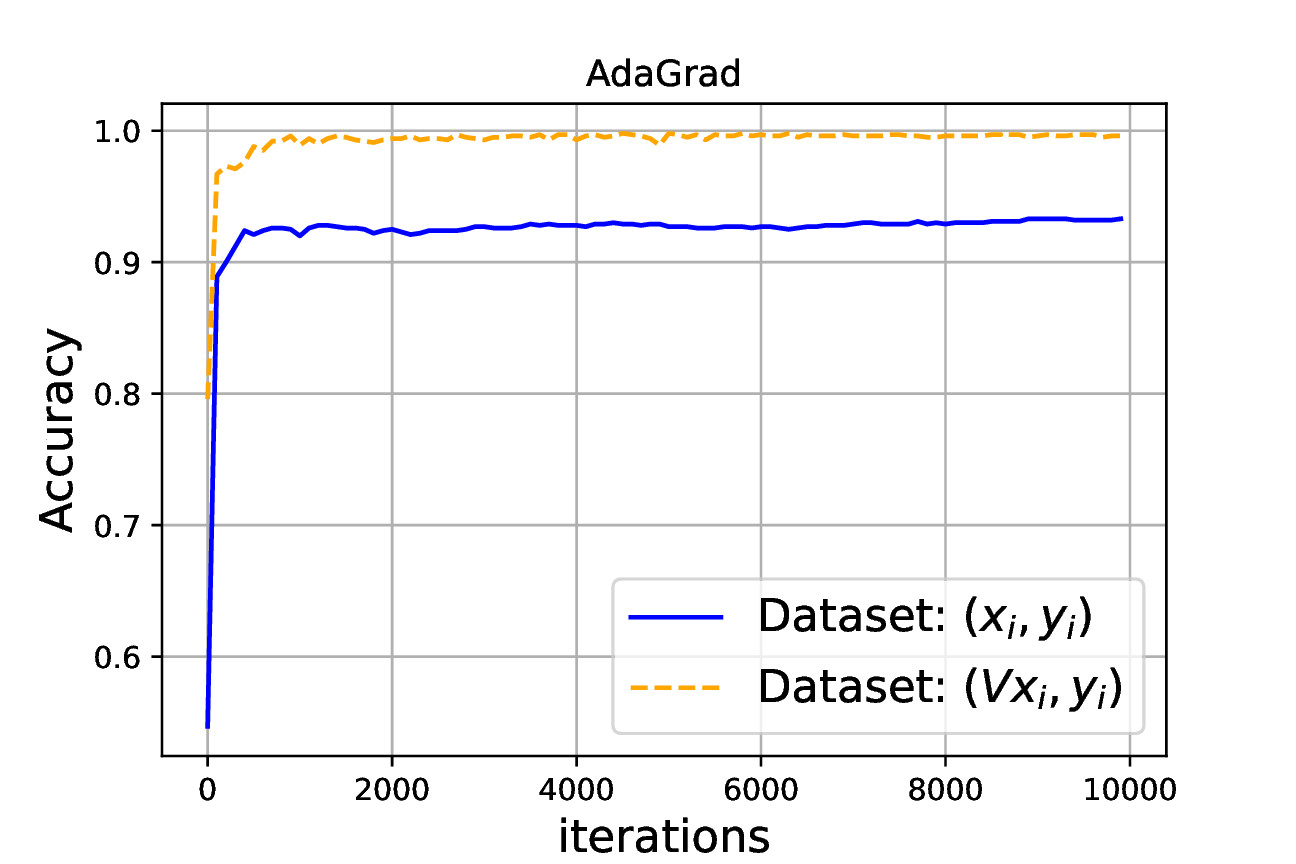}
    \caption{ Plot of Accuracy}\label{fig:scale_accuracy_adagrad}
\end{subfigure}
\begin{subfigure}[b]{0.32\textwidth}
    \centering
    \includegraphics[width=\textwidth]{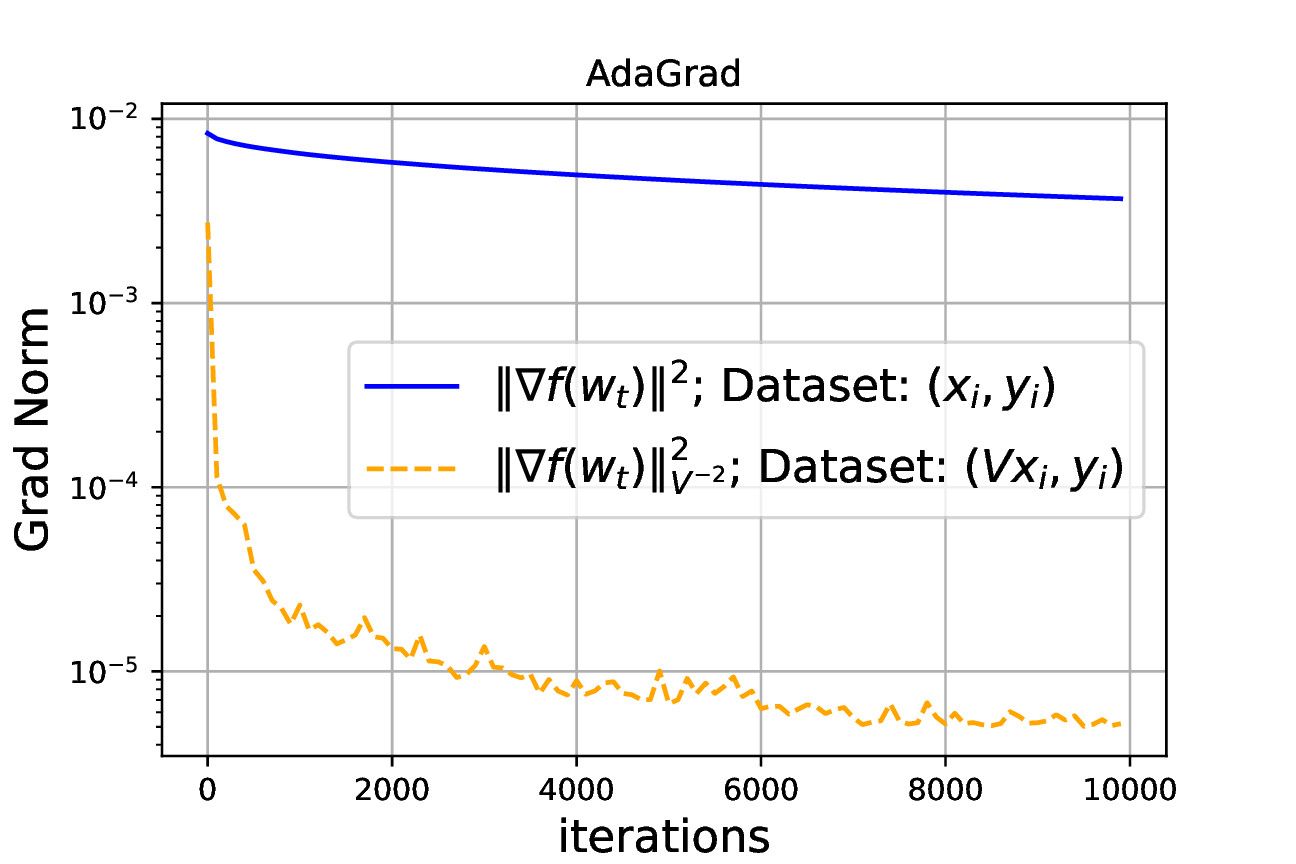}
    \caption{Plot of Gradient Norm}\label{fig:scale_grad_adagrad}
\end{subfigure}
    \caption{Comparison of \algname{\textcolor{PineGreen}{AdaGrad}} on scaled and un-scaled data. Figures \ref{fig:scale_fval_adagrad}, and \ref{fig:scale_accuracy_adagrad} plot the functional value $f(w_t)$ and accuracy on scaled and unscaled data, respectively. Figure \ref{fig:scale_grad_adagrad} plots $\left\| \nabla f(w_t) \right\|^2$ and $\left\| \nabla f(w_t) \right\|_{V^{-2}}^2$ for unscaled and scaled data respectively.} \label{fig:scale_plots_adagrad}
\end{figure}

\newpage

\section*{NeurIPS Paper Checklist}

\begin{enumerate}

\item {\bf Claims}
    \item[] Question: Do the main claims made in the abstract and introduction accurately reflect the paper's contributions and scope?
    \item[] Answer: \answerYes{} 
    \item[] Justification: the new method and its scale invariance property are introduced in Section~\ref{sec:motivation_alg}, main convergence results are provided in Section~\ref{sec:analysis}, and the numerical results are provided in Section~\ref{sec:numerical_results}. 

\item {\bf Limitations}
    \item[] Question: Does the paper discuss the limitations of the work performed by the authors?
    \item[] Answer: \answerYes{} 
    \item[] Justification: see Section~\ref{sec:analysis}.

\item {\bf Theory Assumptions and Proofs}
    \item[] Question: For each theoretical result, does the paper provide the full set of assumptions and a complete (and correct) proof?
    \item[] Answer: \answerYes{} 
    \item[] Justification: see Section~\ref{sec:analysis} and the Appendix.

    \item {\bf Experimental Result Reproducibility}
    \item[] Question: Does the paper fully disclose all the information needed to reproduce the main experimental results of the paper to the extent that it affects the main claims and/or conclusions of the paper (regardless of whether the code and data are provided or not)?
    \item[] Answer: \answerYes{} 
    \item[] Justification: see Section~\ref{sec:numerical_results}.

\item {\bf Open access to data and code}
    \item[] Question: Does the paper provide open access to the data and code, with sufficient instructions to faithfully reproduce the main experimental results, as described in supplemental material?
    \item[] Answer: \answerYes{} 
    \item[] Justification: see Section~\ref{sec:numerical_results}.

\item {\bf Experimental Setting/Details}
    \item[] Question: Does the paper specify all the training and test details (e.g., data splits, hyperparameters, how they were chosen, type of optimizer, etc.) necessary to understand the results?
    \item[] Answer: \answerYes{} 
    \item[] Justification: see Section~\ref{sec:numerical_results}.

\item {\bf Experiment Statistical Significance}
    \item[] Question: Does the paper report error bars suitably and correctly defined or other appropriate information about the statistical significance of the experiments?
    \item[] Answer: \answerNA{} 
    \item[] Justification: the results are consistent for different runs.

\item {\bf Experiments Compute Resources}
    \item[] Question: For each experiment, does the paper provide sufficient information on the computer resources (type of compute workers, memory, time of execution) needed to reproduce the experiments?
    \item[] Answer: \answerYes{} 
    \item[] Justification: We have added all the details in Section \ref{sec:numerical_results}.
    
\item {\bf Code Of Ethics}
    \item[] Question: Does the research conducted in the paper conform, in every respect, with the NeurIPS Code of Ethics \url{https://neurips.cc/public/EthicsGuidelines}?
    \item[] Answer: \answerYes{} 
    \item[] Justification: the paper follows NeurIPS Code of Ethics.

\item {\bf Broader Impacts}
    \item[] Question: Does the paper discuss both potential positive societal impacts and negative societal impacts of the work performed?
    \item[] Answer: \answerNA{} 
    \item[] Justification: the paper is mostly theoretical and does not have a direct societal impact.
    
\item {\bf Safeguards}
    \item[] Question: Does the paper describe safeguards that have been put in place for responsible release of data or models that have a high risk for misuse (e.g., pretrained language models, image generators, or scraped datasets)?
    \item[] Answer: \answerNA{} 
    \item[] Justification: we do not release data or models.

\item {\bf Licenses for existing assets}
    \item[] Question: Are the creators or original owners of assets (e.g., code, data, models), used in the paper, properly credited and are the license and terms of use explicitly mentioned and properly respected?
    \item[] Answer: \answerYes{} 
    \item[] Justification: see Section~\ref{sec:numerical_results}.

\item {\bf New Assets}
    \item[] Question: Are new assets introduced in the paper well documented and is the documentation provided alongside the assets?
    \item[] Answer: \answerNA{} 
    \item[] Justification: not applicable.

\item {\bf Crowdsourcing and Research with Human Subjects}
    \item[] Question: For crowdsourcing experiments and research with human subjects, does the paper include the full text of instructions given to participants and screenshots, if applicable, as well as details about compensation (if any)? 
    \item[] Answer: \answerNA{} 
    \item[] Justification: not applicable.

\item {\bf Institutional Review Board (IRB) Approvals or Equivalent for Research with Human Subjects}
    \item[] Question: Does the paper describe potential risks incurred by study participants, whether such risks were disclosed to the subjects, and whether Institutional Review Board (IRB) approvals (or an equivalent approval/review based on the requirements of your country or institution) were obtained?
    \item[] Answer: \answerNA{} 
    \item[] Justification: not applicable.

\end{enumerate}

\end{document}